\documentclass[11pt,a4paper]{article}
\usepackage[hyperref]{acl2020}

\usepackage{times}
\usepackage{latexsym}

\usepackage{microtype}

\aclfinalcopy %
\def\aclpaperid{3154} %

\usepackage[utf8]{inputenc} %
\usepackage[T1]{fontenc}    %
\usepackage{hyperref}       %
\usepackage{url}            %
\usepackage{booktabs}       %
\usepackage{amsfonts}       %
\usepackage{nicefrac}       %
\usepackage{microtype}      %
\usepackage{graphicx}
\usepackage{float}
\usepackage{amssymb}%
\usepackage{amsmath}
\usepackage{mathtools}
\usepackage{bbm}
\usepackage{amsthm}
\usepackage{xcolor}
\usepackage{pifont}%
\usepackage{subcaption}
\usepackage{tabularx}
\newtheorem{proposition}{Proposition}
\newtheorem*{definition*}{Definition}
\newtheorem{theorem}{Theorem}
\newtheorem{corollary}{Corollary}

\usepackage{multirow}

\usepackage{amsmath,amssymb}

\newcommand\blfootnote[1]{%
  \begingroup
  \renewcommand\thefootnote{}\footnote{#1}%
  \addtocounter{footnote}{-1}%
  \endgroup
}

\title{Rationalizing Text Matching:\\Learning Sparse Alignments via Optimal Transport}
\author{%
  Kyle Swanson\textsuperscript{*} \\
  ASAPP, Inc. \\
  New York, USA \\
  \texttt{kswanson@asapp.com} \\
  \And
  Lili Yu\textsuperscript{*} \\
  ASAPP, Inc. \\
  New York, USA \\
  \texttt{liliyu@asapp.com} \\
  \And
  Tao Lei \\
  ASAPP, Inc. \\
  New York, USA \\
  \texttt{tao@asapp.com} \\
}

\date{}

\begin{document}
\maketitle

\begin{abstract}

Selecting input features of top relevance has become a popular method for building self-explaining models. 
In this work, we extend this selective rationalization approach to text matching, where the goal is to jointly select and align text pieces, such as tokens or sentences, as a justification for the downstream prediction.
Our approach employs optimal transport (OT) to find a minimal cost alignment between the inputs.
However, directly applying OT often produces dense and therefore uninterpretable alignments.
To overcome this limitation, we introduce novel constrained variants of the OT problem that result in highly sparse alignments with controllable sparsity.
Our model is end-to-end differentiable using the Sinkhorn algorithm for OT and can be trained without any alignment annotations.
We evaluate our model on the StackExchange, MultiNews, e-SNLI, and MultiRC datasets.
Our model achieves very sparse rationale selections with high fidelity while preserving prediction accuracy compared to strong attention baseline models.\textsuperscript{\textdagger}
\blfootnote{\textsuperscript{*}Denotes equal contribution.}
\blfootnote{\textsuperscript{\textdagger}Our code is publicly available at \url{https://github.com/asappresearch/rationale-alignment}.}
\end{abstract}

\section{Introduction}

The growing complexity of deep neural networks has given rise to the desire for self-explaining models~\cite{li2016understanding,ribeiro2016should,zhang2016rationale,ross2017right,sundararajan2017axiomatic,davidselfexplain,chen2018learning}. %
In text classification, for instance, one popular method is to design models that can perform classification using only a \textit{rationale}, which is a subset of the text selected from the model input that fully explains the model's prediction~\cite{Lei_2016,bastings2019interpretable,chang19}.
This \textit{selective rationalization} method, often trained to choose a small yet sufficient number of text spans, makes it easy to interpret the model's prediction by examining the selected text.

In contrast to classification, very little progress has been made toward rationalization for text matching models.
The task of text matching encompasses a wide range of downstream applications, such as similar document recommendation~\cite{dos-santos-etal-2015-learning}, question answering~\cite{lee2019latent}, and fact checking~\cite{thorne2018fever}.
Many of these applications can benefit from selecting and comparing information present in the provided documents.
For instance, consider a similar post suggestion in a tech support forum as shown in Figure~\ref{fig:alignment}.
The extracted rationales could provide deeper insights for forum users while also helping human experts validate and improve the model.

\begin{figure}[!t]
    \centering
    \includegraphics[width=3.0in]{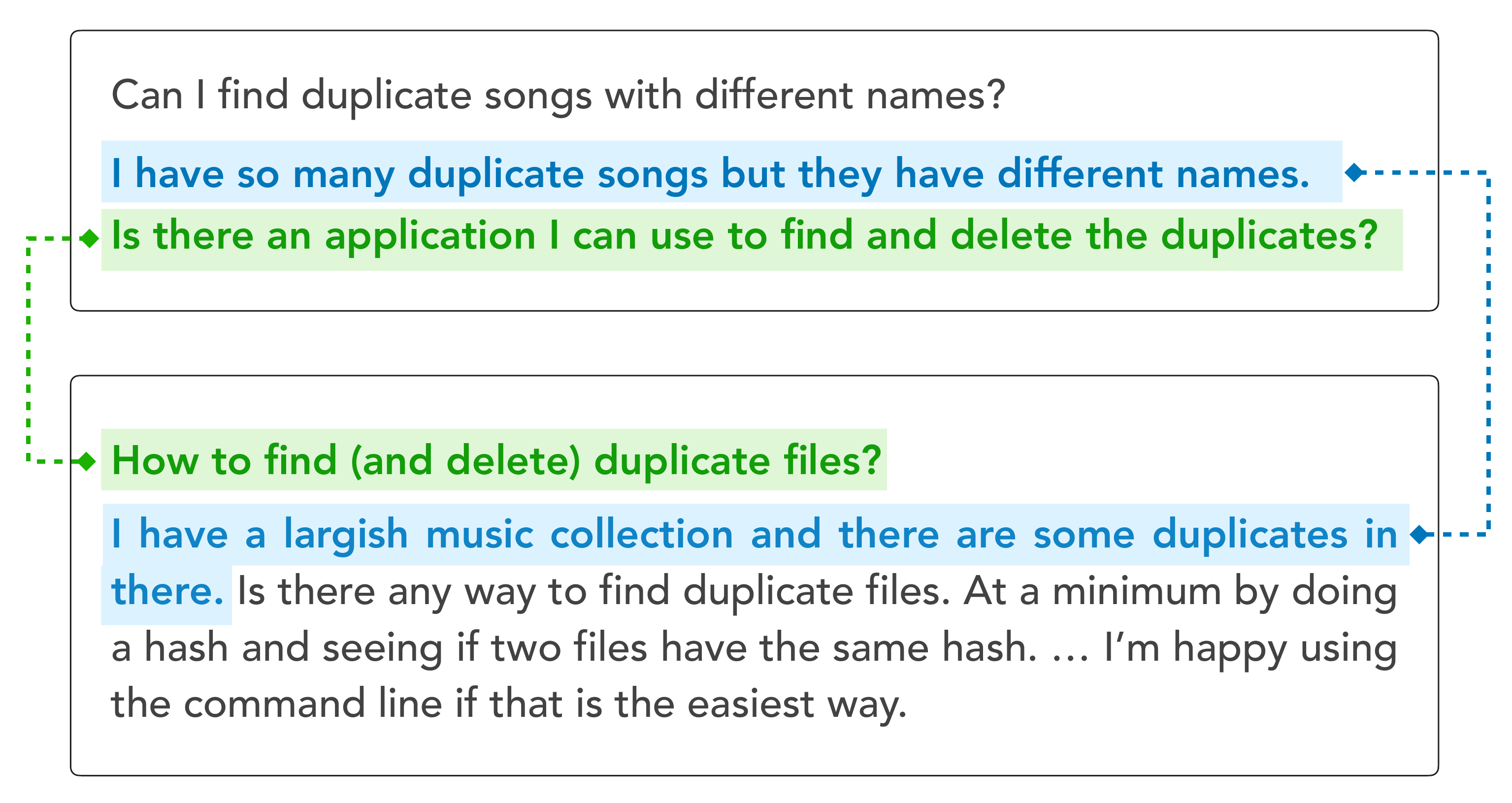}
    \caption{An illustration of a text matching rationale for detecting similar forum posts.}
    \label{fig:alignment}
\end{figure}

In this work, we extend selective rationalization for text matching and focus on two new challenges that are not addressed in previous rationalization work.
First, since text matching is fundamentally about comparing two text documents, rationale selection should be jointly modeled and optimized for matching.
Second, the method should produce an interpretable alignment between the selected rationales showcasing their relations for the downstream prediction.
This is very different from rationalization for text classification, where the selection is performed independently on each input text and an alignment between rationales is unnecessary.

One popular method for aligning inputs is attention-based models~\cite{bahdanau2014neural, rocktaschel2015reasoning,rush2015neural,showattend,kim2017structured}.
However, a limitation of neural attention is that the alignment is rarely sparse, thus making it difficult to interpret how the numerous relations among the text spans lead to the model’s prediction.
Recent work has explored sparse variants of attention~\cite{martins2016softmax,niculae2017regularized,Lin_2018,Malaviya_2018,niculae2018sparsemap}, but the number of non-zero alignments can still be large~\cite{laha2018controllable}.
Additionally, because of the heavy non-linearity following most attention layers, it is difficult to truly attribute the model's predictions to the alignment, which means that attention-based models lack fidelity.

We propose to address these challenges by directly learning sparse yet sufficient alignments using optimal transport (OT).
We use OT as a building block within neural networks for determining the alignment, providing a deeper mathematical justification for the rationale selection.
In order to produce more interpretable rationales, we construct novel variants of OT that have provable and controllable bounds on the sparsity of the alignments.
Selecting and aligning text spans can be jointly optimized within this framework, resulting in optimal text matchings.
Our model is fully end-to-end differentiable using the Sinkhorn algorithm~\cite{cuturi} for OT and can be used with any neural network architecture.

We evaluate our proposed methods on the StackExchange, MultiNews~\cite{multinews}, e-SNLI~\cite{esnli}, and MultiRC~\cite{multirc} datasets, with tasks ranging from similar document identification to reading comprehension.
Compared to other neural baselines, our methods show comparable task performance while selecting only a fraction of the number of alignments.
We further illustrate the effectiveness of our method by analyzing how faithful the model's predictions are to the selected rationales and by comparing the rationales to human-selected rationales provided by  \citet{deyoung2019eraser} on the e-SNLI and MultiRC datasets.

\section{Related Work}

\paragraph{Selective Rationalization.} 
Model interpretability via selective rationalization has attracted considerable interest recently~\cite{Lei_2016,li2016understanding,chen2018learning,chang19}.
Some recent work has focused on overcoming the challenge of learning in the selective rationalization regime, such as by enabling end-to-end differentiable training~\cite{bastings2019interpretable} or by regularizing to avoid performance degeneration~\cite{yu2019}.
Unlike these methods, which perform independent rationale selection on each input document, we extend selective rationalization by jointly learning selection and alignment, as it is better suited for text matching applications.

Concurrent to this work, \citet{deyoung2019eraser} introduce the ERASER benchmark datasets with human-annotated rationales along with several rationalization models. Similarly to \citet{deyoung2019eraser}, we measure the faithfulness of selected rationales, but our work differs in that we additionally emphasize sparsity as a necessary criterion for interpretable alignments.

\paragraph{Alignment.} Models can be made more interpretable by requiring that they explicitly align related elements of the input representation. In NLP, this is often achieved via neural attention~\cite{bahdanau2014neural,chen2015abc,rush2015neural,cheng2016long,decomposableattention,attention1}. 
Many variants of attention, such as temperature-controlled attention~\cite{Lin_2018} and sparsemax~\cite{martins2016softmax}, have been proposed to increase sparsity within the attention weights.
However, it is still debatable whether attention scores are truly explanations~\cite{notattention,notnotattention}.
Distance-based methods of aligning text have also been proposed~\cite{quantum_alignment}, but they similarly cannot guarantee sparsity or explainability.
In this work, we explicitly optimize rationale selection and alignment as an integral part of the model and evaluate the degree to which the alignment explains the model's predictions.

\paragraph{Optimal Transport.} The field of optimal transport (OT) began with \citet{monge}, who explored the problem of determining a minimal cost assignment between sets of equal sizes.
\citet{kantorovich} relaxed Monge's problem to that of determining an optimal transport plan for moving probability mass between two probability distributions.
Since the introduction of a differentiable OT solver by \citet{cuturi}, OT has seen many applications in deep learning and NLP, such as topic embedding ~\cite{ottm}, text generation ~\cite{ottextgeneration}, cross-lingual word embedding alignment \cite{Alvarez_Melis_2018}, graph embedding ~\cite{otgraph}, and learning permutations ~\cite{ot_perm}.
\citet{computational_ot} provides an overview of the computational aspects of OT.
Unlike prior work, we develop novel additional constraints on the OT problem that produce particularly sparse and interpretable alignments.

\section{Problem Formulation}
\label{sec:text_matching}

Consider two related text documents $D^x$ and $D^y$.
These documents are broken down into two sets of text spans, $S^x$ and $S^y$, where the text spans can be words, sentences, paragraphs, or any other chunking of text.
These text spans are then mapped to vector representations using a function $g(\cdot)$ (e.g., a neural network), which produces two sets of vectors representing the inputs, $X = \{ \mathbf{x}_i \}_{i=1}^n = \{ g(S_i^x) \}_{i=1}^n$ and $Y = \{ \mathbf{y}_i \}_{i=1}^m = \{ g(S_i^y) \}_{i=1}^m$, where $\mathbf{x}_i, \mathbf{y}_i \in \mathbb{R}^d$.

We define an \textit{interpretable text matching} as an alignment between the text spans in $X$ and $Y$ that explains the downstream prediction.
Following common practice for previous self-explaining models~\cite{Lei_2016,davidselfexplain}, we specify that a desirable model must produce alignments satisfying the following criteria of interpretability.

\paragraph{Explicitness.}
The alignment between text spans generated by the model should be an observable and understandable component of the model.
Our model explicitly encodes the alignment between $X$ and $Y$ as a matrix $\mathbf{P} \in \mathbb{R}_+^{n \times m}$ where $\mathbf{P}_{i,j}$ indicates the degree to which $\mathbf{x}_i$ and $\mathbf{y}_j$ are aligned. 

\paragraph{Sparsity.}
In order for the alignment to be interpretable, the alignment matrix $\mathbf{P}$ must be sparse, meaning there are very few non-zero alignments between the text spans.
A sparser alignment is easier to interpret as fewer alignments between text spans need to be examined.

\paragraph{Faithfulness.}
An interpretable text matching is only meaningful if the model's predictions are faithful to the alignment, meaning the predictions are directly dependent on it.
Similarly to previous work, our model achieves faithfulness by using only the selected text spans (and their representations) for prediction.
That is, the selected rationales and alignment should be \textit{sufficient} to make accurate predictions.
In addition to sufficiency, faithfulness also requires that the model output should be easily \textit{attributed} to the choice of alignment\footnote{For example, a linear model achieves strong attribution because the importance of each input feature is a constant parameter.}.
For simple attribution, we define our model output as either a linear function of the alignment $\mathbf{P}$ or a shallow feed-forward network on top of $\mathbf{P}$.

In the following sections, we introduce optimal transport as a method to produce interpretable text matchings satisfying all three desiderata.

\section{Background: Optimal Transport}
\label{sec:ot}

An instance of the discrete optimal transport problem consists of two point sets, $X = \{ \mathbf{x}_i \}_{i=1}^n$ and $Y = \{ \mathbf{y}_i \}_{i=1}^m$, with $\mathbf{x}_i, \mathbf{y}_i \in \mathbb{R}^d$. Additionally, $X$ and $Y$ are associated with probability distributions $\mathbf{a} \in \Sigma_n$ and $\mathbf{b} \in \Sigma_m$, respectively, where $\Sigma_n$ is the probability simplex $\Sigma_n \coloneqq \left \{ \mathbf{p} \in \mathbb{R}_+^n : \sum_{i=1}^n \mathbf{p}_i = 1 \right \}$.
A cost function $c(\mathbf{x},\mathbf{y}) : \mathbb{R}^d \times \mathbb{R}^d \to \mathbb{R}$ specifies the cost of aligning a pair of points $\mathbf{x}$ and $\mathbf{y}$. The costs of aligning all pairs of points are summarized by the cost matrix $\mathbf{C} \in \mathbb{R}^{n \times m}$, where $\mathbf{C}_{i,j} = c(\mathbf{x}_i, \mathbf{y}_j)$.

The goal of optimal transport is to compute a mapping that moves probability mass from the points of $X$ (distributed according to $\mathbf{a}$) to the points of $Y$ (distributed according to $\mathbf{b}$) so that the total cost of moving the mass between points is minimized according to the cost function $c$.
This mapping is represented by a transport plan, or alignment matrix, $\mathbf{P} \in \mathbb{R}_+^{n \times m}$, where $\mathbf{P}_{i,j}$ indicates the amount of probability mass moved from $\mathbf{x}_i$ to $\mathbf{y}_j$.
The space of valid alignment matrices is the set
\begin{equation*} 
\begin{split}
    \mathbf{U}(\mathbf{a}, \mathbf{b}) \coloneqq \{ \mathbf{P} \in \mathbb{R}_+^{n \times m} : \mathbf{P} \mathbbm{1}_m = \mathbf{a}, \mathbf{P}^{\text{T}} \mathbbm{1}_n = \mathbf{b} \}
\end{split}
\end{equation*}
since $\mathbf{P}$ must marginalize out to the corresponding probability distributions $\mathbf{a}$ and $\mathbf{b}$ over $X$ and $Y$.

Under this formulation, the optimal transport problem is to find the alignment matrix $\mathbf{P}$ that minimizes the sum of costs weighted by the alignments:
\begin{equation*}
    \label{eq:optimal_coupling}
    L_\mathbf{C}(\mathbf{a}, \mathbf{b}) \coloneqq \underset{\mathbf{P} \in \mathbf{U}(\mathbf{a}, \mathbf{b})}{\mathrm{min}} \langle \mathbf{C}, \mathbf{P} \rangle = \sum_{i,j} \mathbf{C}_{i,j} \mathbf{P}_{i,j}\, .
\end{equation*}
Note that this optimization is a linear programming problem over the convex set $\mathbf{U}(\mathbf{a},\mathbf{b})$.
As a result, one of the extreme points of $\mathbf{U}(\mathbf{a}, \mathbf{b})$ must be an optimal solution.

\subsection{Sparsity Guarantees}

Optimal transport is known to produce alignments that are especially sparse.
In particular, the following propositions characterize the extreme point solution $\mathbf{P}^*$ of $L_\mathbf{C}(\mathbf{a},\mathbf{b})$ and will be important in designing interpretable alignments in Section \ref{sec:ot_constrained}.
\begin{proposition}[\citet{brualdi}, Thm. 8.1.2]
    \label{prop:sparsity}
    Any extreme point $\mathbf{P}^*$ that solves $L_\mathbf{C}(\mathbf{a}, \mathbf{b})$ has at most $n + m - 1$ non-zero entries.
\end{proposition}

\begin{proposition}[\citet{birkhoff}]
    \label{prop:permutation}
    If $n=m$ and $\mathbf{a} = \mathbf{b} = \mathbbm{1}_n / n$, then every extreme point of $\mathbf{U}(\mathbf{a}, \mathbf{b})$ is a permutation matrix.
\end{proposition}

In other words, while the total number of possible aligned pairs is $n\times m$, the optimal alignment $\mathbf{P}^*$ has $\mathcal{O}(n+m)$ non-zero entries.
Furthermore, if $n=m$, then any extreme point solution $\mathbf{P}^*$ is a permutation matrix and thus only has $\mathcal{O}(n)$ non-zero entries.
Figure \ref{fig:alignment_vs_assignment} illustrates two alignments, including one that is a permutation matrix.

Note that the optimal solution of $L_\mathbf{C}(\mathbf{a},\mathbf{b})$ may not be unique in degenerate cases, such as when $\mathbf{C}_{i,j}$ is the same for all $i, j$.
In such degenerate cases, any convex combination of optimal extreme points is a solution.
However, it is possible to modify any OT solver to guarantee that it finds an extreme point (i.e., sparse) solution.
We provide a proof in Appendix~\ref{ssec:unique_solution}, although experimentally we find that these modifications are unnecessary as we nearly always obtain an extreme point solution.

\begin{figure}
    \centering
    \begin{subfigure}[b]{0.45\linewidth}
        \centering
        \includegraphics[width=0.4\textwidth]{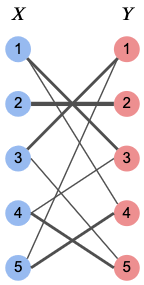}
        \caption{Alignment 1 (graph)}
        \label{fig:alignment_graph}
    \end{subfigure}
    \hfill
    \begin{subfigure}[b]{0.45\linewidth}  
        \centering 
        \includegraphics[width=0.4\textwidth]{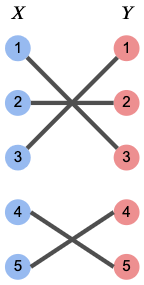}
        \caption{Alignment 2 (graph)}  
        \label{fig:assignment_graph}
    \end{subfigure}
    \vskip\baselineskip
    \begin{subfigure}[b]{0.45\linewidth}   
        \centering 
        \includegraphics[width=0.7\textwidth]{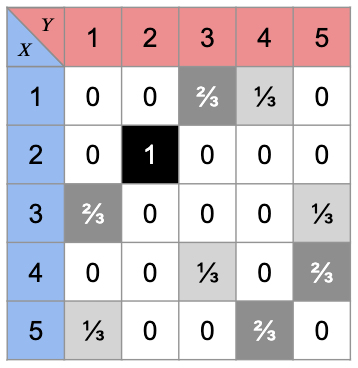}
        \caption{Alignment 1 (matrix)}   
        \label{fig:alignment_matrix}
    \end{subfigure}
    \quad
    \begin{subfigure}[b]{0.45\linewidth}   
        \centering 
        \includegraphics[width=0.7\textwidth]{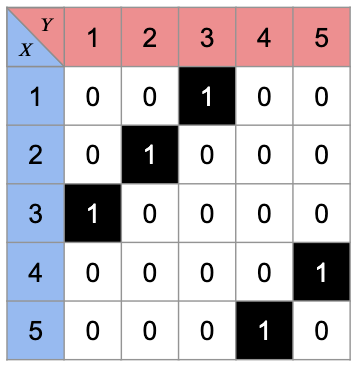}
        \caption{Alignment 2 (matrix)}    
        \label{fig:assignment_matrix}
    \end{subfigure}
    \caption{An illustration of two different alignments between the points of $X$ and $Y$, displayed both as a graph (top) and as an (unnormalized) alignment matrix $\mathbf{P}$ (bottom). Alignment 2 (right) corresponds to the special case where $\mathbf{P}$ is a permutation matrix, which produces an assignment between points in $X$ and $Y$.}
    \label{fig:alignment_vs_assignment}
\end{figure}

\subsection{Sinkhorn Algorithm}
\label{ssec:sinkhorn}

$L_{\mathbf{C}}(\mathbf{a}, \mathbf{b})$ is a linear programming problem and can be solved exactly with interior point methods.
Recently, \citet{cuturi} proposed an entropy-regularized objective that can be solved using a fully differentiable, iterative algorithm, making it ideal for deep learning applications.
Specifically, the entropy-regularized objective is
\begin{equation*}
    \label{eq:entropy_optimal_coupling}
    L_\mathbf{C}^\epsilon(\mathbf{a}, \mathbf{b}) \coloneqq \underset{\mathbf{P} \in \mathbf{U}(\mathbf{a}, \mathbf{b})}{\mathrm{min}} \langle \mathbf{C}, \mathbf{P} \rangle - \epsilon \mathbf{H}(\mathbf{P}),
\end{equation*}
where $\mathbf{H}(\mathbf{P})$ is the entropy of alignment matrix $\mathbf{P}$ and $\epsilon > 0$ controls the amount of entropy regularization. 
In practice, $\epsilon$ can be set sufficiently small such that the solution to $L_{\mathbf{C}}^\epsilon(\mathbf{a}, \mathbf{b})$ is a good approximation of the solution to $L_{\mathbf{C}}(\mathbf{a}, \mathbf{b})$.

Conveniently, $L_{\mathbf{C}}^\epsilon(\mathbf{a}, \mathbf{b})$ has a solution of the form $\mathbf{P}^* = \mathrm{diag}(\mathbf{u})\ \mathbf{K}\ \mathrm{diag}(\mathbf{v})$, where $\mathbf{K} = e^{-\mathbf{C} / \epsilon}$ and $(\mathbf{u}, \mathbf{v}) \in \mathbb{R}_+^n \times \mathbb{R}_+^m$.
The vectors $\mathbf{u}$ and $\mathbf{v}$ can be determined using the Sinkhorn-Knopp matrix scaling algorithm \cite{sinkhorn_knopp}, which iteratively computes
\begin{equation*}
    \mathbf{u} \leftarrow \mathbf{a} \oslash \mathbf{K} \mathbf{v}
    \quad \text{and} \quad 
    \mathbf{v} \leftarrow \mathbf{b} \oslash \mathbf{K}^\mathrm{T} \mathbf{u}
\end{equation*}
where $\oslash$ denotes element-wise division.

Since each iteration consists only of matrix operations, the Sinkhorn algorithm can be used as a differentiable building block in deep learning models.
For instance, in this work we take $\mathbf{C}$ as the distance between hidden representations given by a parameterized neural network encoder.
Our model performs the Sinkhorn iterations until convergence (or a maximum number of steps) and then outputs the alignment $\mathbf{P}$ and the total cost $\left<\mathbf{C}, \mathbf{P}\right>$ as inputs to subsequent components of the model.

\section{Learning Interpretable Alignments}
\label{sec:ot_constrained}

\begin{figure*}
    \centering
    \includegraphics[width=4in]{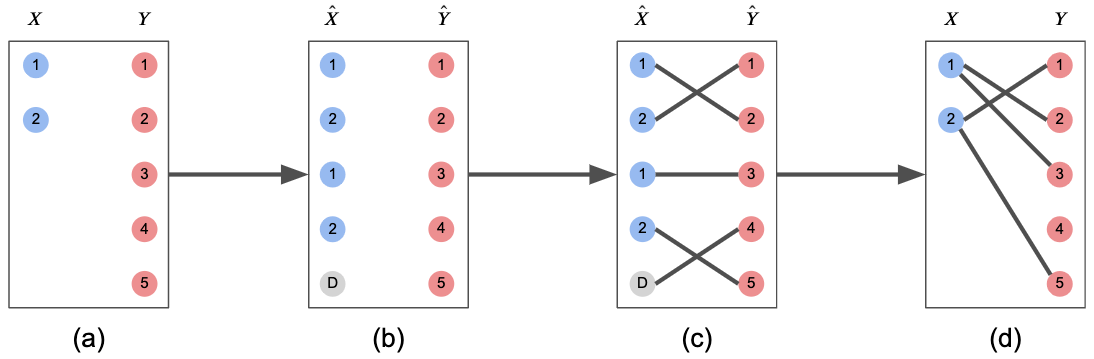}
    \caption{An illustration of the process of computing a one-to-two assignment between the points of $X$ and $Y$. (a) The original points of $X$ and $Y$. (b) $\hat{X}$ and $\hat{Y}$ are constructed so that $\hat{X}$ has two copies of each point in $X$ and one dummy point and $\hat{Y} = Y$. (c) OT is applied to $\hat{X}$ and $\hat{Y}$ using uniform distributions $\mathbf{a}$ and $\mathbf{b}$, which produces a one-to-one assignment between $\hat{X}$ and $\hat{Y}$. (d) A one-to-two assignment between $X$ and $Y$ is extracted from the one-to-one assignment between $\hat{X}$ and $\hat{Y}$.}
    \label{fig:one_to_two_assignment}
\end{figure*}

Using ``vanilla'' OT produces sparse alignments as guaranteed by Proposition~\ref{prop:sparsity}, but the level of sparsity is insufficient to be interpretable.
For instance, Alignment 1 in Figure \ref{fig:alignment_vs_assignment} still has a significant number of non-zero alignment values.
Motivated by this limitation, we propose to encourage greater sparsity and interpretability by constructing OT problems with additional constraints.

\paragraph{General Recipe for Additional Constraints.}
Intuitively, an interpretable alignment should be sparse in two ways.
First, each text span should be aligned to one or a very small number of spans in the other input text. 
Second, the total number of aligned pairs should be small enough so that the alignment can be easily examined by a human.
We modify the OT problem in several ways to guarantee both aspects of sparsity.

We start by forcing the solution to be an \textit{assignment}, which is a one-to-one (or one-to-few) alignment such that every non-zero entry in the alignment matrix is equal, thereby simplifying interpretability.
Alignment 2 in Figure \ref{fig:alignment_vs_assignment} is an example of a one-to-one assignment.
We also consider two other constructions, one that makes every text span in the alignment optional and another that directly limits the total number of aligned pairs.

At the core of our construction are two types of auxiliary points that are added to the input point sets $X$ and $Y$: 
\begin{itemize}
    \item \textbf{Replica points} are exact copies of the original points in $X$ or $Y$ and can be used to control the sparsity of each point's alignment.
    \item \textbf{Dummy points}, also known as tariff-free reservoirs in prior work, are points that can be aligned to with 0 cost.
    Dummy points are used for absorbing unused probability mass in partial transport, where the constraints are relaxed to $\mathbf{P} \mathbbm{1}_m \leq \mathbf{a}$ and $\mathbf{P}^{\text{T}} \mathbbm{1}_n \leq \mathbf{b}$~\cite{caffarelli,figalli}.
\end{itemize}

The idea is to add an appropriate number of replica points and dummy points to create $\hat{X}$ and $\hat{Y}$ with $|\hat{X}| = |\hat{Y}| = N$ for some $N$.
Then by using uniform probability distributions $\mathbf{a} = \mathbf{b} = \mathbbm{1}_N / N$, Proposition \ref{prop:permutation} implies that one of the solutions to the OT problem will be a permutation matrix, i.e., a one-to-one assignment between the points in $\hat{X}$ and $\hat{Y}$.
Since the points of $X$ and $Y$ are included in $\hat{X}$ and $\hat{Y}$, we can directly extract an assignment between $X$ and $Y$ from the assignment between $\hat{X}$ and $\hat{Y}$. Figure \ref{fig:one_to_two_assignment} illustrates the procedure.
Note that the same solution can be attained without explicitly replicating any points by adjusting the probability distributions $\mathbf{a}$ and $\mathbf{b}$, but we use replication for ease of exposition.
Also note that the Sinkhorn algorithm is compatible with replica and dummy points and the model remains differentiable.

We now describe three specific instances of this procedure that produce interpretable assignments with different sparsity patterns.
Without loss of generality, we assume that $n = |X| \leq |Y| = m$.

\paragraph{One-to-$k$ Assignment.}
In this assignment, every point in the smaller set $X$ should map to $k$ points in the larger set $Y$, where $k \in \{1, 2, \dots, \lfloor \frac{m}{n} \rfloor \}$.
This will result in a sparsity of $kn \leq \lfloor \frac{m}{n} \rfloor n \leq m$.

To compute such an assignment, we set $\hat{Y} = Y$ and we construct $\hat{X}$ with $k$ copies of every point in $X$ along with $m - kn$ dummy points.
Since $|\hat{X}| = |\hat{Y}| = m$, applying OT to $\hat{X}$ and $\hat{Y}$ produces a one-to-one assignment between $\hat{X}$ and $\hat{Y}$.
As $\hat{X}$ contains $k$ replicas of each point in $X$, each unique point in $X$ is mapped to $k$ points in $Y$, thus producing a one-to-$k$ assignment.
The remaining $m-kn$ mappings to dummy points are ignored.

\paragraph{Relaxed One-to-$k$ Assignment.}
In a relaxed one-to-$k$ assignment, each point in $X$ can map to \textit{at most} $k$ points in $Y$.
As with the one-to-$k$ assignment, we use $k$ replicas of each point in $X$, but now we add $m$ dummy points to $X$ and $kn$ dummy points to $Y$, meaning $|\hat{X}| = |\hat{Y}| = m + kn$.
Because of the number of replicas, this will produce at most a one-to-$k$ assignment between $X$ and $Y$.
However, since there is now one dummy point in $\hat{Y}$ for every original point in $\hat{X}$, every original point has the option of aligning to a dummy point, resulting in at most $k$ alignments.
Note that in this case, the cost function must take \textit{both} positive and negative values to prevent all original points from mapping to the zero-cost dummy points.

\begin{table}[!t!]
\resizebox{0.47\textwidth}{!}{
\centering
\scriptsize
\begin{tabular} {l|c|c|c|c} 
    \toprule
    Constraint & \#\,R of $X$ & \#\,D in $X'$ & \#\,D in $Y'$ & Sparsity ($s$) \\
    \midrule
    Vanilla & 1 & 0 & 0 & $s \leq n + m - 1$ \\
    One-to-$k$ & $k$ & $m - kn$ & 0 & $s = kn \leq m$ \\
    R one-to-$k$ & $k$ & $m$ & $kn$ & $s \leq kn \leq m$ \\
    Exact-$k$ & 1 & $m - k$ & $n - k$ & $s = k \leq n$ \\
    \bottomrule
\end{tabular}}
\caption{Summary of constrained alignment construction and sparsity. \#\,R is the number of replicas, \#\,D is the number of dummy points, R one-to-$k$ is the relaxed one-to-$k$ assignment, and $n = |X| \leq |Y| = m$.}
\label{table:constraints}
\end{table}

\paragraph{Exact-$k$ Assignment.}
An exact-$k$ assignment maps exactly $k$ points in $X$ to points in $Y$, where $k \leq n$. 
An exact-$k$ assignment can be constructed by adding $m - k$ dummy points to $X$ and $n - k$ dummy points to $Y$, meaning $|\hat{X}| = |\hat{Y}| = n + m - k$.
In this case, the cost function must be strictly positive so that original points map to dummy points whenever possible.
This leaves exactly $k$ alignments between original points in $X$ and $Y$.

\paragraph{Controllable Sparsity.}

Table \ref{table:constraints} summarizes the differences between vanilla OT and the constrained variants.
The freedom to select the type of constraint and the value of $k$ gives fine-grained control over the level of sparsity.
We evaluate the performance of all these variants in our experiments.

\section{Experimental Setup}

\paragraph{Datasets.}
We evaluate our model and all baselines on four benchmarks: two document similarity tasks, MultiNews and StackExchange, and two classification tasks, e-SNLI and MultiRC.
The e-SNLI and MultiRC tasks come from the ERASER benchmark~\cite{deyoung2019eraser}, which was created to evaluate selective rationalization models.
We chose those two datasets as they are best suited for our text matching setup.

StackExchange\footnote{\url{https://stackexchange.com/sites}} is an online question answering platform and has been used as a benchmark in previous work~\cite{dos-santos-etal-2015-learning,shah2018adversarial,perkins-yang-2019-dialog}. 
We took the June 2019 data dumps\footnote{\url{https://archive.org/details/stackexchange}} of the AskUbuntu and SuperUser subdomains of the platform and combined them to form our dataset.

MultiNews~\cite{multinews} is a multi-document summarization dataset where 2 to 10 news articles share a single summary.
We consider every pair of articles that share a summary to be a similar document pair.
Table~\ref{tab:data_stats} shows summary statistics of the two document ranking datasets.

e-SNLI~\cite{esnli} is an extended version of the SNLI dataset~\cite{snli} for natural language inference where the goal is to predict the textual entailment relation (entailment, neutral, or contradiction) between premise and hypothesis sentences.
Human rationales are provided as highlighted words in the two sentences.

MultiRC~\cite{multirc} is a reading comprehension dataset with the goal of assigning a label of true or false to a question-answer pair depending on information from a multi-sentence document. 
We treat the concatenated question and answer as one input and the document as the other input for text matching.
Human rationales are provided as highlighted sentences in the document.

For StackExchange and MultiNews, we split the documents into 80\% train, 10\% validation, and 10\% test, while for e-SNLI and MultiRC, we use the splits from \citet{deyoung2019eraser}.

\begin{table}[!t!]
\centering
\footnotesize
\begin{tabular}{l|c|c} 
\toprule
Metric                       & StackExchange       & MultiNews \\
\midrule
\# docs       &     730,818       &       10,130              \\
\# similar doc pairs   &    187,377     &      22,623      \\
Avg sents per doc    &        3.7       &        31                     \\
Max sents per doc    &      54      &       1,632                     \\
Avg words per doc     &      87      &         680                    \\
Vocab size           &    603,801    &            299,732        \\
\bottomrule
\end{tabular}
\caption{Statistics for the document ranking datasets.}
\label{tab:data_stats}
\end{table}

\begin{figure*}[!t!h]
    \centering
    \includegraphics[width=0.85\textwidth]{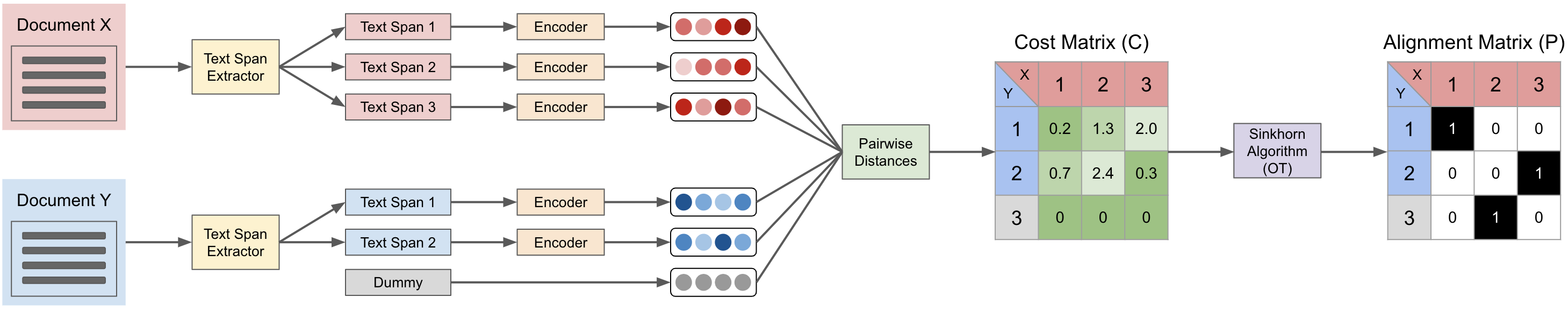}
    \caption{An illustration of our constrained OT model applied to two text documents. The final output of the model depends on a combination of the encodings, the cost matrix, and the alignment matrix.}
    \label{fig:ot_model}
\end{figure*}

\paragraph{Metrics.}
We evaluate models according to the following three criteria.

\begin{enumerate}
    \item \textbf{Sparsity.} To evaluate sparsity, we compute the average percentage of \textit{active} alignments produced by each model, where an alignment is active if it exceeds a small threshold $\lambda$. This threshold is necessary to account for numerical imprecision in alignment values that are essentially zero. We set $\lambda = \frac{0.01}{n \times m}$ unless otherwise specified, where $n$ and $m$ are the number of text spans in the two documents.
    \item \textbf{Sufficiency.} If a model makes a correct prediction given only the rationales, then the rationales are sufficient. We evaluate sufficiency by providing the model only with active alignments and the aligned text representations and by masking non-active inputs (using the threshold $\lambda$).
    \item \textbf{Relevance.} The relevance of rationales is determined by whether a human would deem them valid and relevant. We compute relevance using the token-level F1 scores of model-generated rationales compared to human-selected rationales on the e-SNLI and MultiRC datasets. We also perform a qualitative human evaluation.
\end{enumerate}

\paragraph{Baselines and Implementation Details.}
We use the decomposable attention model \cite{decomposableattention} as our baseline attention model. 
In addition, we compare our model to two attention variants that are designed to encourage sparsity.
The temperature attention variant applies a temperature term $T$ in the softmax operator \cite{Lin_2018}.
The sparse attention variant adopts the sparsemax operator~\cite{martins2016softmax} in place of softmax to produce sparse attention masks.

Our constrained OT model operates as illustrated in Figure~\ref{fig:ot_model}.
After splitting the input documents into sentences, our model independently encodes each sentence and computes pairwise costs between the encoded representations\footnote{For the e-SNLI dataset, where documents are single sentences, we use the contextualized token representations from the output of the sentence encoder following previous work~\cite{thorne2019generating}.}.
Dummy and replica encodings are added as needed for the desired type of constrained alignment.
Our model then applies OT via the Sinkhorn algorithm to the cost matrix $\mathbf{C}$ to produce an optimal alignment matrix $\mathbf{P}$.
For the document ranking tasks, the final score is simply $\left<\mathbf{C}, \mathbf{P}\right>$.
For the classification tasks, we use the alignment $\mathbf{P}$ as a sparse mask to select encoded text representations, and we feed the aggregated representation to a shallow network to predict the output label, similar to our baseline attention models.

For a fair comparison, our models and all baselines use the same neural encoder to encode text spans before the attention or OT operation is applied.
Specifically, we use RoBERTa~\cite{roberta}, a state-of-the-art pre-trained encoder, for the StackExchange and MultiRC dataset. We use use bi-directional recurrent encoders~\cite{lei2017simple} for the MultiNews and e-SNLI datasets\footnote{The input text in the MultiNews dataset is too long for large BERT models. The e-SNLI dataset in ERASER contains human-annotated rationales at the word level while BERT models use sub-word tokenization.}. 
The value of $k$ for the OT constraints is chosen for each dataset by visually inspecting alignments in the validation set, though model performance is robust to the choice of $k$.
In order to compare our models' rationales to human annotations, we use a binary thresholding procedure as described in Appendix \ref{ssec:impl_details}.
We report results averaged over 3 independent runs for each model. Additional implementation details are provided in Appendix~\ref{ssec:impl_details}.

\begin{table*}[!t!]
\centering
\small
\begin{tabular} {l|c|c|c|c|c|c|c|c|c|c} 
\toprule
   & \multicolumn{5}{c|}{ {StackExchange}} & \multicolumn{5}{c}{ {MultiNews} }    \\
\midrule
Model                                                & AUC   & MAP   & MRR   & P@1   & \# Align. & AUC   & MAP   & MRR   & P@1   & \# Align. \\
\midrule
\textrm{OT}                & 98.0 & 91.2 & 91.5 & 86.1 & 8  & 97.5 & {96.8} & {98.1} & {97.2} & 48  \\
\textrm{OT} (1:1)             & 97.7 & 89.7 & 90.0   & 83.9 & 4  & {97.8} & 96.7 & 97.9 & 96.8 & 19  \\
\textrm{OT} (relaxed 1:1) &  97.8 & 88.5 & 88.9 & 81.8 & 3  & 93.1 & 93.2 & 96.0   & 94.1 & 19  \\
\textrm{OT} (exact $k$) & {98.1} & {92.3} & {92.5} & {87.8} & {2}  & 96.4 & 96.3 & 97.7 & 96.6 & {6}   \\
\midrule
\textrm{Attention}                  & {98.2} & 92.4 & 92.5 & 88.0   & 23 & 97.8 & 96.4 & 97.6 & 96.3 & 637 \\
\textrm{Attention} ($T=0.1$)          & 98.2 & 92.4 & 92.5 & 87.7 & 22 & 98.0   & 97.0   & 98.1 & 97.1 & 634 \\
\textrm{Attention} ($T=0.01$)         & 97.9 & 89.7 & 89.9 & 83.5 & {8} & 97.9 & 96.9 & 98.0   & 97.0   & 594  \\
\textrm{Sparse Attention}        & 98.0   & {92.5} & {92.6} & {88.3} & 19  & {98.2} & {97.7} & {98.1} & {97.1} & 330  \\
\bottomrule
\end{tabular}
\caption{Performance of all models on the StackExchange and MultiNews datasets. We report ranking results and the average number of active alignments (\# Align.) used. For our method with the exact $k$ alignment constraint, we set $k=2$ for StackExchange and $k=6$ for MultiNews, respectively.}
\label{table:results}
\end{table*}

\begin{figure}[!t!]
    \centering
    \includegraphics[width=0.95\linewidth]{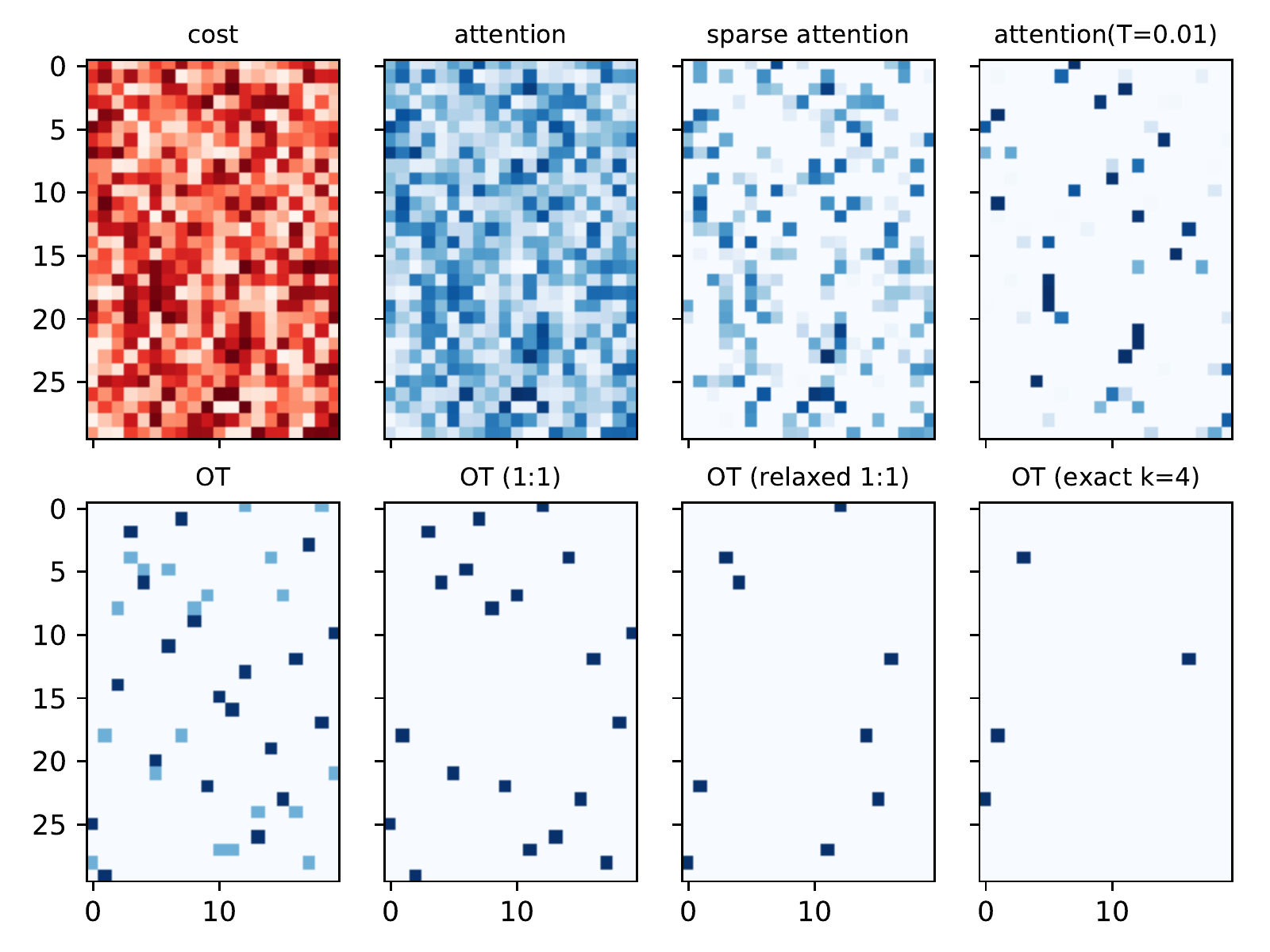}
    \caption{Attention or alignment heatmaps generated by different methods on a synthetic $30\times 20$ cost matrix.} %
    \label{fig:heatmap}
\end{figure}

\section{Results}

\paragraph{Synthetic Visualizations.}
Before experimenting with the datasets, we first analyze the alignments obtained by different methods on a synthetic cost matrix in Figure~\ref{fig:heatmap}.
As shown in the figure, all attention baselines struggle to produce sufficiently sparse alignments, even with the use of a small temperature or the sparsemax operator.
In contrast, our methods are very sparse, as a result of the provable sparsity guarantees of the constrained alignment problem. 
For instance, the relaxed one-to-$k$ assignment produces fewer active alignments than either the number of rows or columns, and the exact-$k$ assignment finds exactly $k=4$ alignments.

\paragraph{StackExchange \& MultiNews.}
Table~\ref{table:results} presents the results of all models on the StackExchange and MultiNews datasets.
We report standard ranking and retrieval metrics including area under the curve (AUC), mean average precision (MAP), mean reciprocal rank (MRR), and precision at 1 (P@1).
The results highlight the ability of our methods to obtain high interpretability while retaining ranking performance comparable to strong attention baselines.
For example, our model is able to use only 6 aligned pairs to achieve a P@1 of 96.6 on the MultiNews dataset.
In comparison, the sparse attention model obtains a P@1 of 97.1 but uses more than 300 alignment pairs and is thus difficult to interpret.
Model complexity and speed on the StackExchange dataset are reported in Table \ref{table:size_and_speed} in Appendix \ref{ssec:impl_details}.

\begin{table*}[th]
\centering
\small
\begin{tabular}{l|c|c|c|@{~~~~}c@{~~~~}|c|c}
\toprule
Model & Accuracy & Task F1 & \% Token & Premise F1 & Hypothesis F1 & P\&H F1 \\
\midrule
OT (relaxed 1:1)         &  82.4  & 82.4  & 69.1 & 25.1 & 43.7 & 34.6 \\
OT (exact $k=4$)       &  81.4 & 81.4 & 38.7 & 24.3 & 45.0 & 35.4   \\
OT (exact $k=3$)       &  81.3 & 81.4   & 29.6 & 28.6 & 50.0 & 39.8   \\
OT (exact $k=2$)       &  81.3 & 81.3 & 21.6 & 24.8 & 30.6 & 27.8 \\
\midrule
Attention         & 76.3 (82.1) & 76.2 & 37.9 & 26.6 & 37.6 & 32.2 \\
Attention ($T=0.1$)    & 73.9 (81.5) & 73.9 & 33.0 & 28.4 & 44.1 & 36.5 \\
Attention ($T=0.01$)    & 70.2 (81.4) & 69.9 & 30.6 & 26.1 & 38.0 & 32.2   \\
Sparse Attention     & 63.5 (75.0) & 63.1  & 12.5  & 8.8  & 24.5 & 17.2 \\
\midrule
\citet{thorne2019generating} & - (81.0) & - & - & 22.2 & 57.8 & - \\
\textsuperscript{\textdagger}\citet{Lei_2016} & - & 90.3 &  - & - & - & 37.9 \\
\toprule
\textsuperscript{\textdagger}\citet{Lei_2016} (+S) & - & 91.7 & - & - & - & 69.2 \\
\textsuperscript{\textdagger}Bert-To-Bert (+S) & - & 73.3 & - & - & - & 70.1 \\
\bottomrule
\end{tabular}
\caption{e-SNLI accuracy, macro-averaged task F1, percentage of tokens in active alignments, and token-level F1 of the model-selected rationales compared to human-annotated rationales for the premise, hypothesis, and both (P\&H F1).
Accuracy numbers in parentheses use all attention weights, not just active ones.
(+S) denotes supervised learning of rationales. \textsuperscript{\textdagger} denotes results from \citet{deyoung2019eraser}.
}
\label{table:snli}
\end{table*}

\paragraph{e-SNLI.}
Table~\ref{table:snli} shows model performance on the e-SNLI dataset. 
As with document similarity ranking, we evaluate classification accuracy when the model uses only the active alignments.
This is to ensure faithfulness, meaning the model truly and exclusively uses the rationales to make predictions.
Since attention is not explicitly trained to use only active alignments, we also report the accuracy of attention models when using all attention weights.

As shown in the table, the accuracy of attention methods decreases significantly when we remove attention weights other than those deemed active by the threshold $\lambda$.
In contrast, our model retains high accuracy even with just the active alignments since sparsity is naturally modeled in our contrained optimal transport framework.
Figure~\ref{fig:snli_tradeoff} visualizes the change to model accuracy when different proportions of tokens are selected by the models. 

\begin{figure}[!t!]
    \centering
    \includegraphics[width=0.95\linewidth]{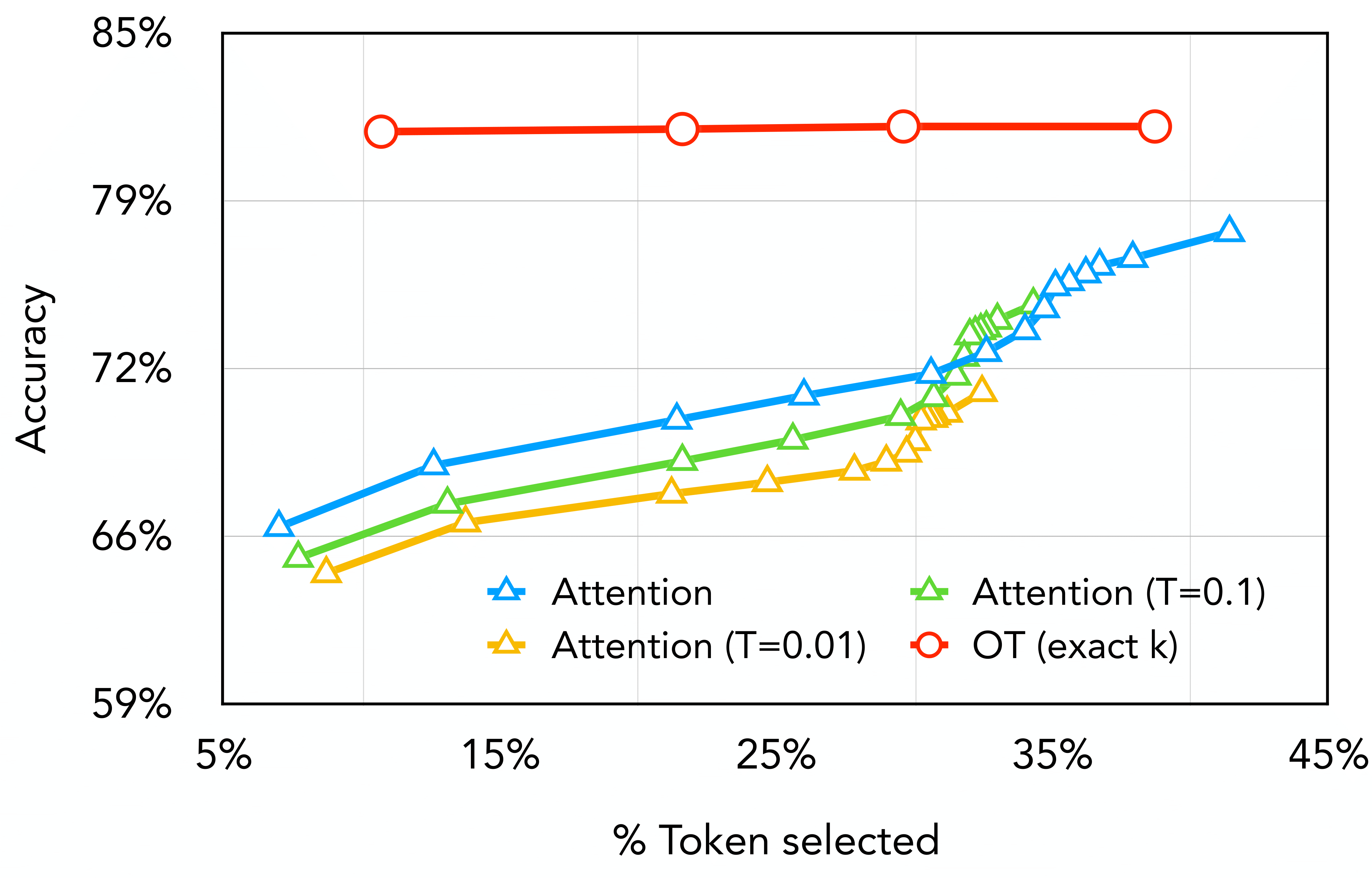}
    \caption{Model accuracy on the e-SNLI dataset when using different percentages of tokens as rationales. The attention model values are obtained using different thresholds $\lambda$ to clip the attention weights while the values for our exact-$k$ model correspond to $k=1, 2, 3, 4$.} %
    \label{fig:snli_tradeoff}
\end{figure}

Table~\ref{table:snli} also presents the token-level F1 scores for the models' selected rationales compared to human-annotated rationales.
Note that the rationale annotations for this task are designed for token selection rather than alignment and are sometimes only on one of the input sentences.
Nevertheless, our model obtains F1 scores on par with recent work~\cite{deyoung2019eraser,thorne2019generating}.

\paragraph{MultiRC.}
Table~\ref{table:multirc} presents the results on the MultiRC dataset.
Compared to attention models, our OT-based models achieve similar task performance with a higher rationale F1 score, despite selecting fewer rationales. 
The model variants from \citet{deyoung2019eraser} in general achieve higher task F1 performance.
However, their unsupervised model suffers from degeneration due to the challenges of end-to-end training without rationale supervision.

We also create supervised versions of our models that learn from the human-annotated rationales during training.
These supervised models achieve comparable task performance to and better rationale F1 scores than models from \citet{deyoung2019eraser}, demonstrating the strength of a sparse rationale alignment.
Supervised training details can be found in Appendix \ref{ssec:impl_details}.

\begin{table}[!t!]
\resizebox{0.47\textwidth}{!}{
\centering
\small
\begin{tabular}{l|c|c|c}
\toprule
Model & Task F1 & \% Token & R. F1 \\
\midrule
OT (1:1)        &  62.3    & 21.6   & 33.7   \\
OT (relaxed 1:1)         &  62.0   & 23.1   & 32.1   \\
OT (relaxed 1:2)         &  62.2    & 24.0   & 35.9   \\
OT (exact $k=2$)         &  62.5    & 25.8   & 34.7   \\
OT (exact $k=3$)         &  62.0    & 24.6   & 37.3   \\
\midrule
Attention                &  62.6    & 44.7  & 21.3  \\
Attention ($T=0.1$)      & 62.6    & 34.7  & 18.2   \\
Attention ($T=0.01$)     & 62.7    & 30.1  & 17.3    \\
Sparse Attention         & 59.3    & 31.3  & 21.2   \\
\midrule
\textsuperscript{\textdagger}\citet{Lei_2016}     & 64.8 & - & 0.0  \\
\toprule
OT (1:1) (+S)     &  61.5    & 19.0   & 50.0     \\
OT (relaxed 1:1) (+S)         &  60.6   & 19.4   & 45.4   \\
OT (relaxed 1:2) (+S)         &   61.5    & 28.7   & 46.8   \\
OT (exact $k=2$) (+S)     &  61.0    & 18.9   & 51.3     \\
OT (exact $k=3$) (+S)    &   60.9   & 23.1   & 49.3   \\
\midrule
\textsuperscript{\textdagger}\citet{Lei_2016} (+S)  & 65.5 & - & 45.6  \\
\textsuperscript{\textdagger}\citet{lehman} (+S)    & 61.4 & - & 14.0  \\
\textsuperscript{\textdagger}Bert-To-Bert (+S)        & 63.3 & - & 41.2  \\
\bottomrule
\end{tabular}}
\caption{MultiRC macro-averaged task F1, percentage of tokens used in active alignments, and token-level F1 of the model-selected rationales compared to human-annotated rationales (R. F1).
(+S) denotes supervised learning of rationales.
\textsuperscript{\textdagger} denotes results from \citet{deyoung2019eraser}. %
\label{table:multirc}
}
\end{table}

\paragraph{Qualitative Studies.}

We performed a human evaluation on documents from StackExchange that reveals that our model's alignments are preferred to attention.
The results of the human evaluation, along with examples of StackExchange and e-SNLI alignments, are provided in Appendix \ref{ssec:qualitative_study}.

\section{Conclusion}
Balancing performance and interpretability in deep learning models has become an increasingly important aspect of model design.
In this work, we propose jointly learning interpretable alignments as part of the downstream prediction to reveal how neural network models operate for text matching applications. 
Our method extends vanilla optimal transport by adding various constraints that produce alignments with highly controllable sparsity patterns, making them particularly interpretable. 
Our models show superiority by selecting very few alignments while achieving text matching performance on par with alternative methods.
As an added benefit, our method is very general in nature and can be used as a differentiable hard-alignment module in larger NLP models that compare two pieces of text, such as sequence-to-sequence models.
Furthermore, our method is agnostic to the underlying nature of the two objects being aligned and can therefore align disparate objects such as images and captions, enabling a wide range of future applications within NLP and beyond.

\section*{Acknowledgments}

We thank Jesse Michel, Derek Chen, Yi Yang, and the anonymous reviewers for their valuable discussions. We thank Sam Altschul, Derek Chen, Amit Ganatra, Alex Lin, James Mullenbach, Jen Seale, Siddharth Varia, and Lei Xu for providing the human evaluation.

\bibliography{refer}

\begin{thebibliography}{64}
\expandafter\ifx\csname natexlab\endcsname\relax\def\natexlab#1{#1}\fi

\bibitem[{Alvarez-Melis and Jaakkola(2018{\natexlab{a}})}]{Alvarez_Melis_2018}
David Alvarez-Melis and Tommi Jaakkola. 2018{\natexlab{a}}.
\newblock \href {https://doi.org/10.18653/v1/D18-1214} {{G}romov-{W}asserstein
  alignment of word embedding spaces}.
\newblock In \emph{Proceedings of the 2018 Conference on Empirical Methods in
  Natural Language Processing}, pages 1881--1890, Brussels, Belgium.
  Association for Computational Linguistics.

\bibitem[{Alvarez-Melis and Jaakkola(2018{\natexlab{b}})}]{davidselfexplain}
David Alvarez-Melis and Tommi Jaakkola. 2018{\natexlab{b}}.
\newblock \href
  {http://papers.nips.cc/paper/8003-towards-robust-interpretability-with-self-explaining-neural-networks.pdf}
  {Towards robust interpretability with self-explaining neural networks}.
\newblock In S.~Bengio, H.~Wallach, H.~Larochelle, K.~Grauman, N.~Cesa-Bianchi,
  and R.~Garnett, editors, \emph{Advances in Neural Information Processing
  Systems 31}, pages 7775--7784. Curran Associates, Inc.

\bibitem[{Bahdanau et~al.(2015)Bahdanau, Cho, and Bengio}]{bahdanau2014neural}
Dzmitry Bahdanau, Kyunghyun Cho, and Yoshua Bengio. 2015.
\newblock \href {https://arxiv.org/abs/1409.0473} {Neural machine translation
  by jointly learning to align and translate}.
\newblock \emph{International Conference on Learning Representations}.

\bibitem[{Bastings et~al.(2019)Bastings, Aziz, and
  Titov}]{bastings2019interpretable}
Joost Bastings, Wilker Aziz, and Ivan Titov. 2019.
\newblock \href {https://doi.org/10.18653/v1/P19-1284} {Interpretable neural
  predictions with differentiable binary variables}.
\newblock In \emph{Proceedings of the 57th Annual Meeting of the Association
  for Computational Linguistics}, pages 2963--2977, Florence, Italy.
  Association for Computational Linguistics.

\bibitem[{Bird et~al.(2009)Bird, Loper, and Klein}]{nltk}
Steven Bird, Edward Loper, and Ewan Klein. 2009.
\newblock \emph{Natural Language Processing with Python}.
\newblock O’Reilly Media Inc.

\bibitem[{Birkhoff(1946)}]{birkhoff}
Garrett Birkhoff. 1946.
\newblock Tres observaciones sobre el algebra lineal.
\newblock \emph{Universidad Nacional de Tucum\'an Revista Series A},
  5:147--151.

\bibitem[{Bojanowski et~al.(2017)Bojanowski, Grave, Joulin, and
  Mikolov}]{Bojanowski_2017}
Piotr Bojanowski, Edouard Grave, Armand Joulin, and Tomas Mikolov. 2017.
\newblock \href {https://doi.org/10.1162/tacl_a_00051} {Enriching word vectors
  with subword information}.
\newblock \emph{Transactions of the Association for Computational Linguistics},
  5:135--146.

\bibitem[{Bowman et~al.(2015)Bowman, Angeli, Potts, and Manning}]{snli}
Samuel~R. Bowman, Gabor Angeli, Christopher Potts, and Christopher~D. Manning.
  2015.
\newblock \href {https://doi.org/10.18653/v1/D15-1075} {A large annotated
  corpus for learning natural language inference}.
\newblock In \emph{Proceedings of the 2015 Conference on Empirical Methods in
  Natural Language Processing}, pages 632--642, Lisbon, Portugal. Association
  for Computational Linguistics.

\bibitem[{Brenier(1987)}]{brenier}
Yann Brenier. 1987.
\newblock D\'ecomposition polaire et r\'earrangement monotone des champs de
  vecteurs.
\newblock \emph{C. R. Acad. Sci. Paris S\'er I Math.}, 305:805--808.

\bibitem[{Brualdi(1982)}]{birkhoff_algorithm}
Richard~A. Brualdi. 1982.
\newblock Notes of the birkhoff algorithm for doubly stochastic matrices.
\newblock \emph{Canadian Mathematical Bulletin}, 25:191--199.

\bibitem[{Brualdi(2006)}]{brualdi}
Richard~A Brualdi. 2006.
\newblock \emph{Combinatorial Matrix Classes}, volume 108.
\newblock Cambridge University Press.

\bibitem[{Caffarelli and McCann(2010)}]{caffarelli}
Luis~A. Caffarelli and Robert~J. McCann. 2010.
\newblock Free boundaries in optimal transport and monge-amp\`ere obstacle
  problems.
\newblock \emph{Annals of Mathematics}, 171:673--730.

\bibitem[{Camburu et~al.(2018)Camburu, Rockt\"{a}schel, Lukasiewicz, and
  Blunsom}]{esnli}
Oana-Maria Camburu, Tim Rockt\"{a}schel, Thomas Lukasiewicz, and Phil Blunsom.
  2018.
\newblock \href
  {http://papers.nips.cc/paper/8163-e-snli-natural-language-inference-with-natural-language-explanations.pdf}
  {e-snli: Natural language inference with natural language explanations}.
\newblock In S.~Bengio, H.~Wallach, H.~Larochelle, K.~Grauman, N.~Cesa-Bianchi,
  and R.~Garnett, editors, \emph{Advances in Neural Information Processing
  Systems 31}, pages 9539--9549. Curran Associates, Inc.

\bibitem[{Chang et~al.(2019)Chang, Zhang, Yu, and Jaakkola}]{chang19}
Shiyu Chang, Yang Zhang, Mo~Yu, and Tommi Jaakkola. 2019.
\newblock \href
  {http://papers.nips.cc/paper/9196-a-game-theoretic-approach-to-class-wise-selective-rationalization.pdf}
  {A game theoretic approach to class-wise selective rationalization}.
\newblock In H.~Wallach, H.~Larochelle, A.~Beygelzimer, F.~d'~Alch\'{e}-Buc,
  E.~Fox, and R.~Garnett, editors, \emph{Advances in Neural Information
  Processing Systems 32}, pages 10055--10065. Curran Associates, Inc.

\bibitem[{Chen et~al.(2018{\natexlab{a}})Chen, Song, Wainwright, and
  Jordan}]{chen2018learning}
Jianbo Chen, Le~Song, Martin Wainwright, and Michael Jordan.
  2018{\natexlab{a}}.
\newblock \href {http://proceedings.mlr.press/v80/chen18j.html} {Learning to
  explain: An information-theoretic perspective on model interpretation}.
\newblock In \emph{Proceedings of the 35th International Conference on Machine
  Learning}, volume~80 of \emph{Proceedings of Machine Learning Research},
  pages 883--892, Stockholmsmässan, Stockholm Sweden. PMLR.

\bibitem[{Chen et~al.(2015)Chen, Wang, Chen, Gao, Xu, and
  Nevatia}]{chen2015abc}
Kan Chen, Jiang Wang, Liang-Chieh Chen, Haoyuan Gao, Wei Xu, and Ram Nevatia.
  2015.
\newblock \href {https://arxiv.org/pdf/1511.05960.pdf} {Abc-cnn: An attention
  based convolutional neural network for visual question answering}.
\newblock \emph{arXiv preprint arXiv:1511.05960}.

\bibitem[{Chen et~al.(2018{\natexlab{b}})Chen, Dai, Tao, Zhang, Gan, Shen,
  Zhang, Wang, Zhang, and Carin}]{ottextgeneration}
Liqun Chen, Shuyang Dai, Chenyang Tao, Haichao Zhang, Zhe Gan, Dinghan Shen,
  Yizhe Zhang, Guoyin Wang, Ruiyi Zhang, and Lawrence Carin.
  2018{\natexlab{b}}.
\newblock \href
  {http://papers.nips.cc/paper/7717-adversarial-text-generation-via-feature-movers-distance.pdf}
  {Adversarial text generation via feature-mover's distance}.
\newblock In S.~Bengio, H.~Wallach, H.~Larochelle, K.~Grauman, N.~Cesa-Bianchi,
  and R.~Garnett, editors, \emph{Advances in Neural Information Processing
  Systems 31}, pages 4666--4677. Curran Associates, Inc.

\bibitem[{Cheng et~al.(2016)Cheng, Dong, and Lapata}]{cheng2016long}
Jianpeng Cheng, Li~Dong, and Mirella Lapata. 2016.
\newblock \href {https://doi.org/10.18653/v1/D16-1053} {Long short-term
  memory-networks for machine reading}.
\newblock In \emph{Proceedings of the 2016 Conference on Empirical Methods in
  Natural Language Processing}, pages 551--561, Austin, Texas. Association for
  Computational Linguistics.

\bibitem[{Cuturi(2013)}]{cuturi}
Marco Cuturi. 2013.
\newblock \href
  {http://papers.nips.cc/paper/4927-sinkhorn-distances-lightspeed-computation-of-optimal-transport.pdf}
  {Sinkhorn distances: Lightspeed computation of optimal transport}.
\newblock In C.~J.~C. Burges, L.~Bottou, M.~Welling, Z.~Ghahramani, and K.~Q.
  Weinberger, editors, \emph{Advances in Neural Information Processing Systems
  26}, pages 2292--2300. Curran Associates, Inc.

\bibitem[{DeYoung et~al.(2019)DeYoung, Jain, Rajani, Lehman, Xiong, Socher, and
  Wallace}]{deyoung2019eraser}
Jay DeYoung, Sarthak Jain, Nazneen~Fatema Rajani, Eric Lehman, Caiming Xiong,
  Richard Socher, and Byron~C. Wallace. 2019.
\newblock \href {http://arxiv.org/abs/1911.03429} {Eraser: A benchmark to
  evaluate rationalized nlp models}.
\newblock \emph{arXiv preprint arXiv:1911.03429}.

\bibitem[{Fabbri et~al.(2019)Fabbri, Li, She, Li, and Radev}]{multinews}
Alexander Fabbri, Irene Li, Tianwei She, Suyi Li, and Dragomir Radev. 2019.
\newblock \href {https://doi.org/10.18653/v1/P19-1102} {Multi-news: A
  large-scale multi-document summarization dataset and abstractive hierarchical
  model}.
\newblock In \emph{Proceedings of the 57th Annual Meeting of the Association
  for Computational Linguistics}, pages 1074--1084, Florence, Italy.
  Association for Computational Linguistics.

\bibitem[{Figalli(2010)}]{figalli}
Alessio Figalli. 2010.
\newblock The optimal partial transport problem.
\newblock \emph{Archive for Rational Mechanics and Analysis}, 195:533--560.

\bibitem[{Jain and Wallace(2019)}]{notattention}
Sarthak Jain and Byron~C. Wallace. 2019.
\newblock \href {http://arxiv.org/abs/1902.10186} {Attention is not
  explanation}.
\newblock \emph{arXiv preprint arXiv:1902.10186}.

\bibitem[{Kantorovich(1942)}]{kantorovich}
Leonid Kantorovich. 1942.
\newblock On the transfer of masses (in russian).
\newblock \emph{Doklady Akademii Nauk}, 37:227--229.

\bibitem[{Khashabi et~al.(2018)Khashabi, Chaturvedi, Roth, Upadhyay, and
  Roth}]{multirc}
Daniel Khashabi, Snigdha Chaturvedi, Michael Roth, Shyam Upadhyay, and Dan
  Roth. 2018.
\newblock \href {https://doi.org/10.18653/v1/N18-1023} {Looking beyond the
  surface: A challenge set for reading comprehension over multiple sentences}.
\newblock In \emph{Proceedings of the 2018 Conference of the North {A}merican
  Chapter of the Association for Computational Linguistics: Human Language
  Technologies, Volume 1 (Long Papers)}, pages 252--262, New Orleans,
  Louisiana. Association for Computational Linguistics.

\bibitem[{Kim et~al.(2018)Kim, Denton, Hoang, and Rush}]{kim2017structured}
Yoon Kim, Carl Denton, Luong Hoang, and Alexander~M Rush. 2018.
\newblock \href {https://openreview.net/forum?id=HkE0Nvqlg} {Structured
  attention networks}.
\newblock \emph{International Conference on Learning Representations}.

\bibitem[{Kingma and Ba(2014)}]{Adam}
Diederik~P. Kingma and Jimmy Ba. 2014.
\newblock \href {https://arxiv.org/abs/1412.6980} {Adam: A method for
  stochastic optimization}.
\newblock \emph{International Conference on Learning Representations}.

\bibitem[{Kusner et~al.(2015)Kusner, Sun, Kolkin, and Weinberger}]{ottm}
Matt Kusner, Yu~Sun, Nicholas Kolkin, and Kilian Weinberger. 2015.
\newblock \href {http://proceedings.mlr.press/v37/kusnerb15.html} {From word
  embeddings to document distances}.
\newblock In \emph{Proceedings of the 32nd International Conference on Machine
  Learning}, volume~37 of \emph{Proceedings of Machine Learning Research},
  pages 957--966, Lille, France. PMLR.

\bibitem[{Laha et~al.(2018)Laha, Chemmengath, Agrawal, Khapra,
  Sankaranarayanan, and Ramaswamy}]{laha2018controllable}
Anirban Laha, Saneem~Ahmed Chemmengath, Priyanka Agrawal, Mitesh Khapra,
  Karthik Sankaranarayanan, and Harish~G Ramaswamy. 2018.
\newblock \href
  {http://papers.nips.cc/paper/7878-on-controllable-sparse-alternatives-to-softmax.pdf}
  {On controllable sparse alternatives to softmax}.
\newblock In S.~Bengio, H.~Wallach, H.~Larochelle, K.~Grauman, N.~Cesa-Bianchi,
  and R.~Garnett, editors, \emph{Advances in Neural Information Processing
  Systems 31}, pages 6422--6432. Curran Associates, Inc.

\bibitem[{Lee et~al.(2019)Lee, Chang, and Toutanova}]{lee2019latent}
Kenton Lee, Ming-Wei Chang, and Kristina Toutanova. 2019.
\newblock \href {https://doi.org/10.18653/v1/P19-1612} {Latent retrieval for
  weakly supervised open domain question answering}.
\newblock In \emph{Proceedings of the 57th Annual Meeting of the Association
  for Computational Linguistics}, pages 6086--6096, Florence, Italy.
  Association for Computational Linguistics.

\bibitem[{Lehman et~al.(2019)Lehman, DeYoung, Barzilay, and Wallace}]{lehman}
Eric Lehman, Jay DeYoung, Regina Barzilay, and Byron~C. Wallace. 2019.
\newblock \href {https://doi.org/10.18653/v1/N19-1371} {Inferring which medical
  treatments work from reports of clinical trials}.
\newblock In \emph{Proceedings of the 2019 Conference of the North {A}merican
  Chapter of the Association for Computational Linguistics: Human Language
  Technologies, Volume 1 (Long and Short Papers)}, pages 3705--3717,
  Minneapolis, Minnesota. Association for Computational Linguistics.

\bibitem[{Lei et~al.(2016)Lei, Barzilay, and Jaakkola}]{Lei_2016}
Tao Lei, Regina Barzilay, and Tommi Jaakkola. 2016.
\newblock \href {https://doi.org/10.18653/v1/D16-1011} {Rationalizing neural
  predictions}.
\newblock In \emph{Proceedings of the 2016 Conference on Empirical Methods in
  Natural Language Processing}, pages 107--117, Austin, Texas. Association for
  Computational Linguistics.

\bibitem[{Lei et~al.(2018)Lei, Zhang, Wang, Dai, and Artzi}]{lei2017simple}
Tao Lei, Yu~Zhang, Sida~I. Wang, Hui Dai, and Yoav Artzi. 2018.
\newblock \href {https://doi.org/10.18653/v1/D18-1477} {Simple recurrent units
  for highly parallelizable recurrence}.
\newblock In \emph{Proceedings of the 2018 Conference on Empirical Methods in
  Natural Language Processing}, pages 4470--4481, Brussels, Belgium.
  Association for Computational Linguistics.

\bibitem[{Li et~al.(2016)Li, Monroe, and Jurafsky}]{li2016understanding}
Jiwei Li, Will Monroe, and Dan Jurafsky. 2016.
\newblock \href {https://arxiv.org/abs/1612.08220} {Understanding neural
  networks through representation erasure}.
\newblock \emph{arXiv preprint arXiv:1612.08220}.

\bibitem[{Li et~al.(2019)Li, Wang, and Melucci}]{quantum_alignment}
Qiuchi Li, Benyou Wang, and Massimo Melucci. 2019.
\newblock \href {https://doi.org/10.18653/v1/N19-1420} {{CNM}: An interpretable
  complex-valued network for matching}.
\newblock In \emph{Proceedings of the 2019 Conference of the North {A}merican
  Chapter of the Association for Computational Linguistics: Human Language
  Technologies, Volume 1 (Long and Short Papers)}, pages 4139--4148,
  Minneapolis, Minnesota. Association for Computational Linguistics.

\bibitem[{Lin et~al.(2018)Lin, Sun, Ren, Li, and Su}]{Lin_2018}
Junyang Lin, Xu~Sun, Xuancheng Ren, Muyu Li, and Qi~Su. 2018.
\newblock \href {https://doi.org/10.18653/v1/D18-1331} {Learning when to
  concentrate or divert attention: Self-adaptive attention temperature for
  neural machine translation}.
\newblock In \emph{Proceedings of the 2018 Conference on Empirical Methods in
  Natural Language Processing}, pages 2985--2990, Brussels, Belgium.
  Association for Computational Linguistics.

\bibitem[{Liu et~al.(2019)Liu, Ott, Goyal, Du, Joshi, Chen, Levy, Lewis,
  Zettlemoyer, and Stoyanov}]{roberta}
Yinhan Liu, Myle Ott, Naman Goyal, Jingfei Du, Mandar Joshi, Danqi Chen, Omer
  Levy, Mike Lewis, Luke Zettlemoyer, and Veselin Stoyanov. 2019.
\newblock \href {http://arxiv.org/abs/arXiv:1907.11692} {Roberta: A robustly
  optimized bert pretraining approach}.
\newblock \emph{arXiv preprint arXiv:1907.11692}.

\bibitem[{Malaviya et~al.(2018)Malaviya, Ferreira, and Martins}]{Malaviya_2018}
Chaitanya Malaviya, Pedro Ferreira, and Andr{\'e} F.~T. Martins. 2018.
\newblock \href {https://doi.org/10.18653/v1/P18-2059} {Sparse and constrained
  attention for neural machine translation}.
\newblock In \emph{Proceedings of the 56th Annual Meeting of the Association
  for Computational Linguistics (Volume 2: Short Papers)}, pages 370--376,
  Melbourne, Australia. Association for Computational Linguistics.

\bibitem[{Martins and Astudillo(2016)}]{martins2016softmax}
Andre Martins and Ramon Astudillo. 2016.
\newblock \href {http://proceedings.mlr.press/v48/martins16.html} {From softmax
  to sparsemax: A sparse model of attention and multi-label classification}.
\newblock In \emph{Proceedings of The 33rd International Conference on Machine
  Learning}, volume~48 of \emph{Proceedings of Machine Learning Research},
  pages 1614--1623, New York, New York, USA. PMLR.

\bibitem[{Mena et~al.(2018)Mena, Belanger, Linderman, and Snoek}]{ot_perm}
Gonzalo Mena, David Belanger, Scott Linderman, and Jasper Snoek. 2018.
\newblock \href {https://openreview.net/forum?id=Byt3oJ-0W} {Learning latent
  permutations with gumbel-sinkhorn networks}.
\newblock \emph{International Conference on Learning Representations}.

\bibitem[{Monge(1781)}]{monge}
Gaspard Monge. 1781.
\newblock M\'emoir sur la th\'eorie des d\'eblais et des remblais.
\newblock \emph{Histoire de l'Acad\'emie Royale des Sciences}, pages 666--704.

\bibitem[{Niculae and Blondel(2017)}]{niculae2017regularized}
Vlad Niculae and Mathieu Blondel. 2017.
\newblock \href
  {http://papers.nips.cc/paper/6926-a-regularized-framework-for-sparse-and-structured-neural-attention.pdf}
  {A regularized framework for sparse and structured neural attention}.
\newblock In I.~Guyon, U.~V. Luxburg, S.~Bengio, H.~Wallach, R.~Fergus,
  S.~Vishwanathan, and R.~Garnett, editors, \emph{Advances in Neural
  Information Processing Systems 30}, pages 3338--3348. Curran Associates, Inc.

\bibitem[{Niculae et~al.(2018)Niculae, Martins, Blondel, and
  Cardie}]{niculae2018sparsemap}
Vlad Niculae, André F.~T. Martins, Mathieu Blondel, and Claire Cardie. 2018.
\newblock \href {https://arxiv.org/abs/1802.04223} {Sparsemap: Differentiable
  sparse structured inference}.
\newblock In \emph{Proceedings of the 35th International Conference on Machine
  Learning}, volume~80, pages 3799--3808. PMLR.

\bibitem[{Parikh et~al.(2016)Parikh, T{\"a}ckstr{\"o}m, Das, and
  Uszkoreit}]{decomposableattention}
Ankur Parikh, Oscar T{\"a}ckstr{\"o}m, Dipanjan Das, and Jakob Uszkoreit. 2016.
\newblock \href {https://doi.org/10.18653/v1/D16-1244} {A decomposable
  attention model for natural language inference}.
\newblock In \emph{Proceedings of the 2016 Conference on Empirical Methods in
  Natural Language Processing}, pages 2249--2255, Austin, Texas. Association
  for Computational Linguistics.

\bibitem[{Paszke et~al.(2017)Paszke, Gross, Chintala, Chanan, Yang, DeVito,
  Lin, Desmaison, Antiga, and Lerer}]{paszke2017}
Adam Paszke, Sam Gross, Soumith Chintala, Gregory Chanan, Edward Yang, Zachary
  DeVito, Zeming Lin, Alban Desmaison, Luca Antiga, and Adam Lerer. 2017.
\newblock \href {https://openreview.net/forum?id=BJJsrmfCZ} {Automatic
  differentiation in pytorch}.
\newblock \emph{NIPS 2017 Autodiff Workshop}.

\bibitem[{Perkins and Yang(2019)}]{perkins-yang-2019-dialog}
Hugh Perkins and Yi~Yang. 2019.
\newblock \href {https://www.aclweb.org/anthology/D19-1413} {Dialog intent
  induction with deep multi-view clustering}.
\newblock In \emph{Proceedings of the 2019 Conference on Empirical Methods in
  Natural Language Processing and the 9th International Joint Conference on
  Natural Language Processing (EMNLP-IJCNLP)}, pages 4014--4023. Association
  for Computational Linguistics.

\bibitem[{Peyr\'e and Cuturi(2019)}]{computational_ot}
Gabriel Peyr\'e and Marco Cuturi. 2019.
\newblock Computational optimal transport.
\newblock \emph{Foundations and Trends in Machine Learning}, 11:335--607.

\bibitem[{Ribeiro et~al.(2016)Ribeiro, Singh, and Guestrin}]{ribeiro2016should}
Marco~Tulio Ribeiro, Sameer Singh, and Carlos Guestrin. 2016.
\newblock \href {https://doi.org/10.1145/2939672.2939778} {“why should i
  trust you?”: Explaining the predictions of any classifier}.
\newblock In \emph{Proceedings of the 22nd ACM SIGKDD International Conference
  on Knowledge Discovery and Data Mining}, KDD ’16, page 1135–1144, New
  York, NY, USA. Association for Computing Machinery.

\bibitem[{Rockt{\"a}schel et~al.(2015)Rockt{\"a}schel, Grefenstette, Hermann,
  Ko{\v{c}}isk{\`y}, and Blunsom}]{rocktaschel2015reasoning}
Tim Rockt{\"a}schel, Edward Grefenstette, Karl~Moritz Hermann, Tom{\'a}{\v{s}}
  Ko{\v{c}}isk{\`y}, and Phil Blunsom. 2015.
\newblock \href {https://arxiv.org/pdf/1509.06664.pdf} {Reasoning about
  entailment with neural attention}.
\newblock \emph{arXiv preprint arXiv:1509.06664}.

\bibitem[{Ross et~al.(2017)Ross, Hughes, and Doshi-Velez}]{ross2017right}
Andrew~Slavin Ross, Michael~C. Hughes, and Finale Doshi-Velez. 2017.
\newblock \href {https://doi.org/10.24963/ijcai.2017/371} {Right for the right
  reasons: Training differentiable models by constraining their explanations}.
\newblock In \emph{Proceedings of the Twenty-Sixth International Joint
  Conference on Artificial Intelligence, {IJCAI-17}}, pages 2662--2670.

\bibitem[{Rush et~al.(2015)Rush, Chopra, and Weston}]{rush2015neural}
Alexander~M. Rush, Sumit Chopra, and Jason Weston. 2015.
\newblock \href {https://doi.org/10.18653/v1/D15-1044} {A neural attention
  model for abstractive sentence summarization}.
\newblock In \emph{Proceedings of the 2015 Conference on Empirical Methods in
  Natural Language Processing}, pages 379--389, Lisbon, Portugal. Association
  for Computational Linguistics.

\bibitem[{dos Santos et~al.(2015)dos Santos, Barbosa, Bogdanova, and
  Zadrozny}]{dos-santos-etal-2015-learning}
C{\'\i}cero dos Santos, Luciano Barbosa, Dasha Bogdanova, and Bianca Zadrozny.
  2015.
\newblock \href {https://doi.org/10.3115/v1/P15-2114} {Learning hybrid
  representations to retrieve semantically equivalent questions}.
\newblock In \emph{Proceedings of the 53rd Annual Meeting of the Association
  for Computational Linguistics and the 7th International Joint Conference on
  Natural Language Processing (Volume 2: Short Papers)}, pages 694--699,
  Beijing, China. Association for Computational Linguistics.

\bibitem[{Schmitzer(2016)}]{stable_sinkhorn}
Bernhard Schmitzer. 2016.
\newblock \href {https://doi.org/10.1137/16M1106018} {Stabilized sparse scaling
  algorithms for entropy regularized transport problems}.
\newblock \emph{SIAM Journal on Scientific Computing}, 41:A1443--A1481.

\bibitem[{Shah et~al.(2018)Shah, Lei, Moschitti, Romeo, and
  Nakov}]{shah2018adversarial}
Darsh Shah, Tao Lei, Alessandro Moschitti, Salvatore Romeo, and Preslav Nakov.
  2018.
\newblock \href {https://doi.org/10.18653/v1/D18-1131} {Adversarial domain
  adaptation for duplicate question detection}.
\newblock In \emph{Proceedings of the 2018 Conference on Empirical Methods in
  Natural Language Processing}, pages 1056--1063, Brussels, Belgium.
  Association for Computational Linguistics.

\bibitem[{Sinkhorn and Knopp(1967)}]{sinkhorn_knopp}
Richard Sinkhorn and Paul Knopp. 1967.
\newblock Concerning nonnegative matrices and doubly stochastic matrices.
\newblock \emph{Pacific J. Math}, 21:343--348.

\bibitem[{Sundararajan et~al.(2017)Sundararajan, Taly, and
  Yan}]{sundararajan2017axiomatic}
Mukund Sundararajan, Ankur Taly, and Qiqi Yan. 2017.
\newblock \href {https://dl.acm.org/doi/10.5555/3305890.3306024} {Axiomatic
  attribution for deep networks}.
\newblock In \emph{Proceedings of the 34th International Conference on Machine
  Learning - Volume 70}, ICML’17, page 3319–3328. JMLR.org.

\bibitem[{Thorne et~al.(2018)Thorne, Vlachos, Christodoulopoulos, and
  Mittal}]{thorne2018fever}
James Thorne, Andreas Vlachos, Christos Christodoulopoulos, and Arpit Mittal.
  2018.
\newblock \href {https://doi.org/10.18653/v1/N18-1074} {{FEVER}: a large-scale
  dataset for fact extraction and {VER}ification}.
\newblock In \emph{Proceedings of the 2018 Conference of the North {A}merican
  Chapter of the Association for Computational Linguistics: Human Language
  Technologies, Volume 1 (Long Papers)}, pages 809--819, New Orleans,
  Louisiana. Association for Computational Linguistics.

\bibitem[{Thorne et~al.(2019)Thorne, Vlachos, Christodoulopoulos, and
  Mittal}]{thorne2019generating}
James Thorne, Andreas Vlachos, Christos Christodoulopoulos, and Arpit Mittal.
  2019.
\newblock \href {https://doi.org/10.18653/v1/N19-1101} {Generating token-level
  explanations for natural language inference}.
\newblock In \emph{Proceedings of the 2019 Conference of the North {A}merican
  Chapter of the Association for Computational Linguistics: Human Language
  Technologies, Volume 1 (Long and Short Papers)}, pages 963--969, Minneapolis,
  Minnesota. Association for Computational Linguistics.

\bibitem[{Wiegreffe and Pinter(2019)}]{notnotattention}
Sarah Wiegreffe and Yuval Pinter. 2019.
\newblock \href {https://doi.org/10.18653/v1/D19-1002} {Attention is not not
  explanation}.
\newblock In \emph{Proceedings of the 2019 Conference on Empirical Methods in
  Natural Language Processing and the 9th International Joint Conference on
  Natural Language Processing (EMNLP-IJCNLP)}, pages 11--20, Hong Kong, China.
  Association for Computational Linguistics.

\bibitem[{Xie et~al.(2017)Xie, Ma, Dai, and Hovy}]{attention1}
Qizhe Xie, Xuezhe Ma, Zihang Dai, and Eduard Hovy. 2017.
\newblock \href {https://doi.org/10.18653/v1/P17-1088} {An interpretable
  knowledge transfer model for knowledge base completion}.
\newblock In \emph{Proceedings of the 55th Annual Meeting of the Association
  for Computational Linguistics (Volume 1: Long Papers)}, pages 950--962,
  Vancouver, Canada. Association for Computational Linguistics.

\bibitem[{Xu et~al.(2019)Xu, Luo, Zha, and Duke}]{otgraph}
Hongteng Xu, Dixin Luo, Hongyuan Zha, and Lawrence~Carin Duke. 2019.
\newblock \href {http://proceedings.mlr.press/v97/xu19b.html}
  {Gromov-{W}asserstein learning for graph matching and node embedding}.
\newblock In \emph{Proceedings of the 36th International Conference on Machine
  Learning}, volume~97 of \emph{Proceedings of Machine Learning Research},
  pages 6932--6941, Long Beach, California, USA. PMLR.

\bibitem[{Xu et~al.(2015)Xu, Ba, Kiros, Cho, Courville, Salakhutdinov, Zemel,
  and Bengio}]{showattend}
Kelvin Xu, Jimmy~Lei Ba, Ryan Kiros, Kyunghyun Cho, Aaron Courville, Ruslan
  Salakhutdinov, Richard~S. Zemel, and Yoshua Bengio. 2015.
\newblock \href {https://dl.acm.org/doi/10.5555/3045118.3045336} {Show, attend
  and tell: Neural image caption generation with visual attention}.
\newblock In \emph{Proceedings of the 32nd International Conference on
  International Conference on Machine Learning - Volume 37}, ICML’15, page
  2048–2057. JMLR.org.

\bibitem[{Yu et~al.(2019)Yu, Chang, Zhang, and Jaakkola}]{yu2019}
Mo~Yu, Shiyu Chang, Yang Zhang, and Tommi Jaakkola. 2019.
\newblock \href {https://doi.org/10.18653/v1/D19-1420} {Rethinking cooperative
  rationalization: Introspective extraction and complement control}.
\newblock In \emph{Proceedings of the 2019 Conference on Empirical Methods in
  Natural Language Processing and the 9th International Joint Conference on
  Natural Language Processing (EMNLP-IJCNLP)}, pages 4094--4103, Hong Kong,
  China. Association for Computational Linguistics.

\bibitem[{Zhang et~al.(2016)Zhang, Marshall, and Wallace}]{zhang2016rationale}
Ye~Zhang, Iain Marshall, and Byron~C. Wallace. 2016.
\newblock \href {https://doi.org/10.18653/v1/D16-1076} {Rationale-augmented
  convolutional neural networks for text classification}.
\newblock In \emph{Proceedings of the 2016 Conference on Empirical Methods in
  Natural Language Processing}, pages 795--804, Austin, Texas. Association for
  Computational Linguistics.

\end{thebibliography}
\bibliographystyle{acl_natbib}

\appendix

\begin{center}
    \large{\textbf{Appendix}}
    \label{sec:supplemental}
\end{center}

\section{Qualitative Study}
\label{ssec:qualitative_study}

\paragraph{Human Evaluation.}

We performed a human evaluation of rationale quality on the StackExchange dataset.
We asked $8$ annotators to rate $270$ rationale examples selected from three models including OT (exact $k=2$), Attention ($T=0.01$), and Sparse Attention. 
For each example, we presented the human annotator with a pair of similar documents along with the extracted alignment rationales.
The annotator then assigned a score of 0, 1, or 2 for each of the following categories: redundancy, relevance, and overall quality.
A higher score is always better (i.e., less redundant, more relevant, higher overall quality).
For attention-based models, we selected the top 2 or 3 aligned pairs (according to the attention weights) such that the number of pairs is similar to that of the OT (exact $k=2$) model.
The results are shown in Figure~\ref{fig:human_eval}.
Attention models have more redundancy as well as higher relevance. This is not surprising since selecting redundant alignments can result in fewer mistakes. 
In comparison, our OT-based model achieves much less redundancy and a better overall score.

\paragraph{Example Rationales.}
Figure~\ref{fig:examples} shows examples of rationales generated from our OT (exact $k=2$) model on the StackExchange dataset. 
Our extracted rationales effectively identify sentences with similar semantic meaning and capture the major topics in the AskUbuntu subdomain. 
Figure~\ref{fig:examples_snli} similarly shows example rationales on the e-SNLI dataset.

\begin{figure}[!t]
    \centering
    \begin{subfigure}[b]{\linewidth}
        \centering
        \includegraphics[width=0.8\textwidth]{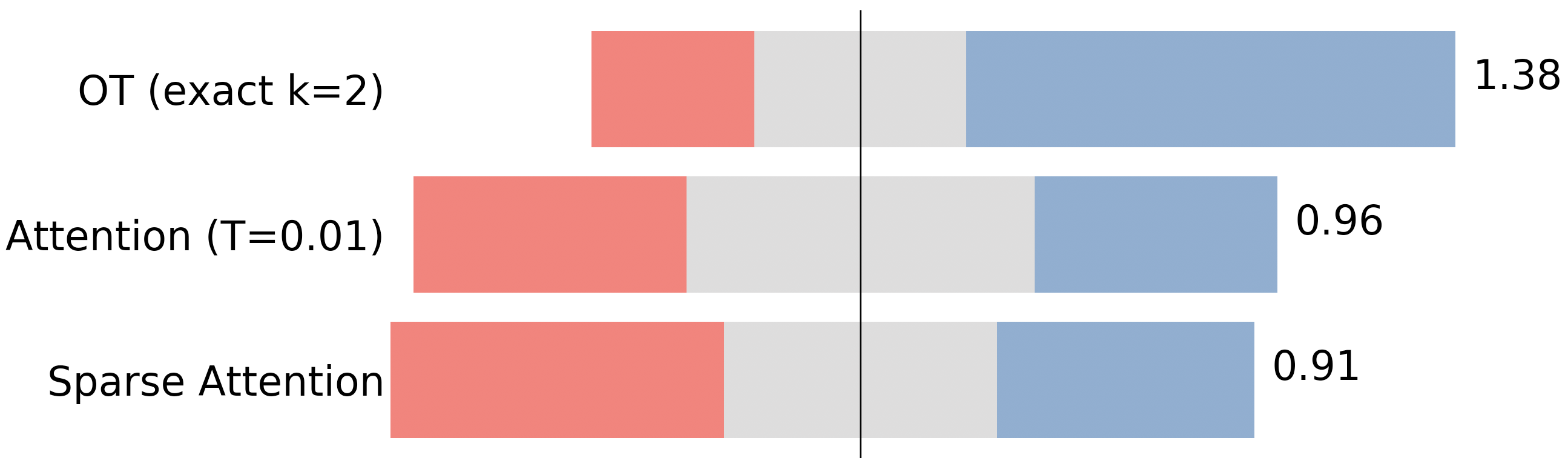}
        \caption{Redundancy}
        \label{fig:redundancy}
    \end{subfigure}
    \vskip\baselineskip
    \begin{subfigure}[b]{\linewidth}   
        \centering 
        \includegraphics[width=0.8\textwidth]{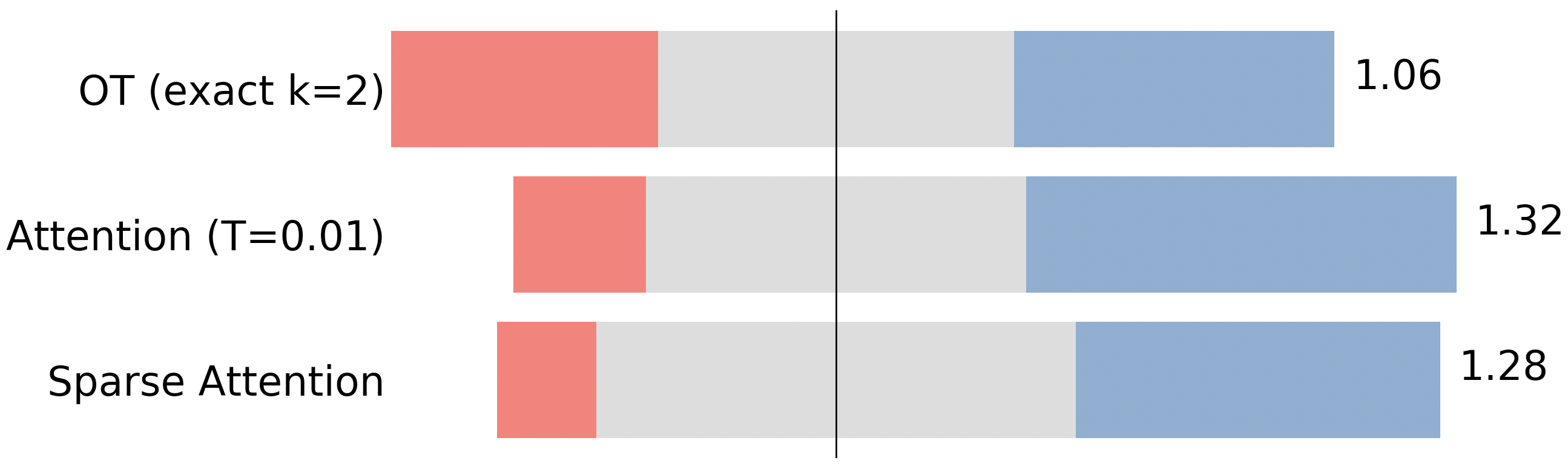}
        \caption{Relevance}    
        \label{fig:relevance}
    \end{subfigure}
        \vskip\baselineskip
    \begin{subfigure}[b]{\linewidth}   
        \centering 
        \includegraphics[width=0.8\textwidth]{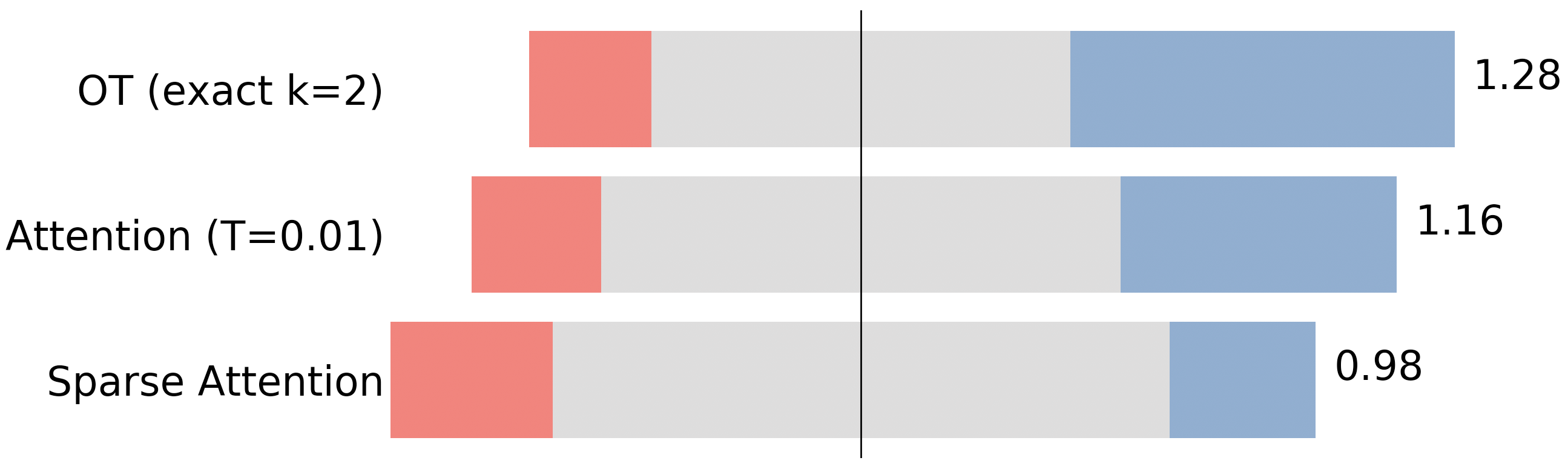}
        \caption{Overall quality}    
        \label{fig:overall}
    \end{subfigure}
    \caption{Human evaluation of rationales extracted from StackExchange document pairs using metrics of redundancy, relevance, and overall quality. Scores are either 0 (red), 1 (gray), or 2 (blue) and higher is better. The length of each bar segment indicates the proportion of examples with that score, and the number to the right of each bar is the average score.}
    \label{fig:human_eval}
\end{figure}

\begin{figure*}[!t]
    \begin{tabular}{cc}
        \includegraphics[height=2.0in]{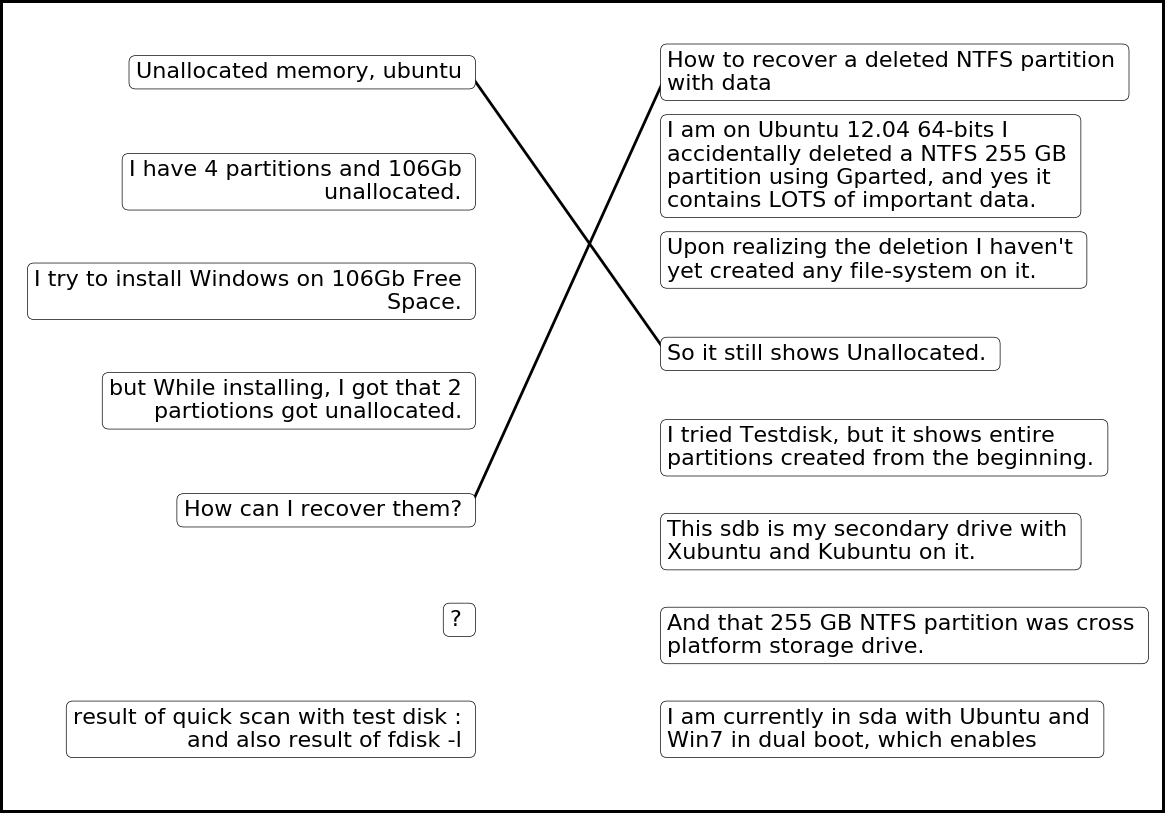}
        & 
        \includegraphics[height=2.0in]{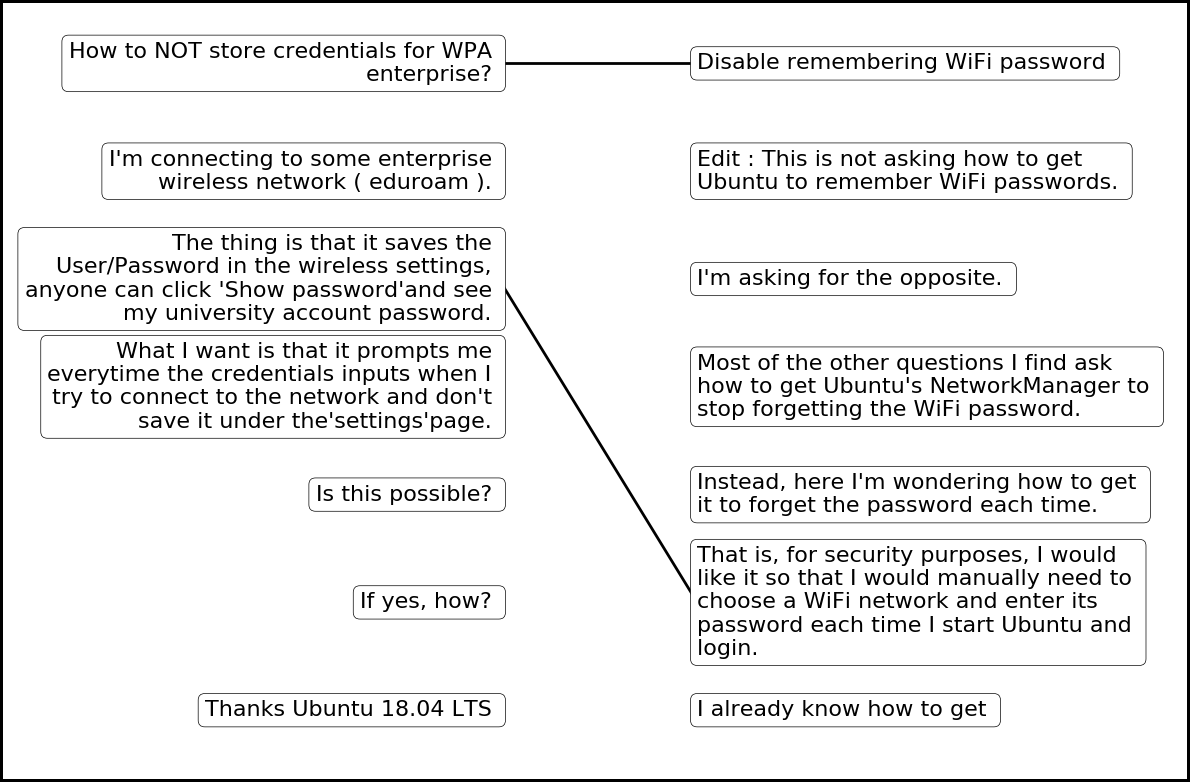}
        \\ 
        \includegraphics[height=1.95in]{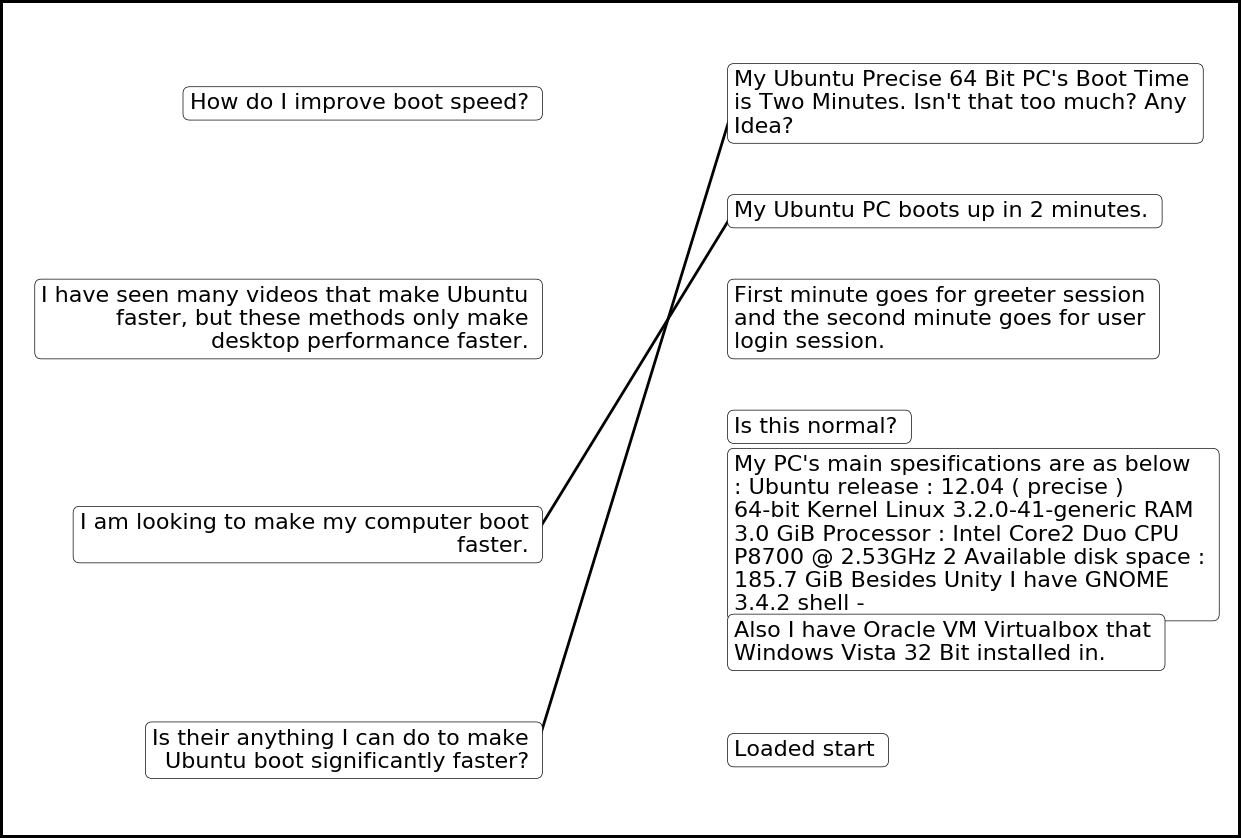}
        &
        \includegraphics[height=1.95in]{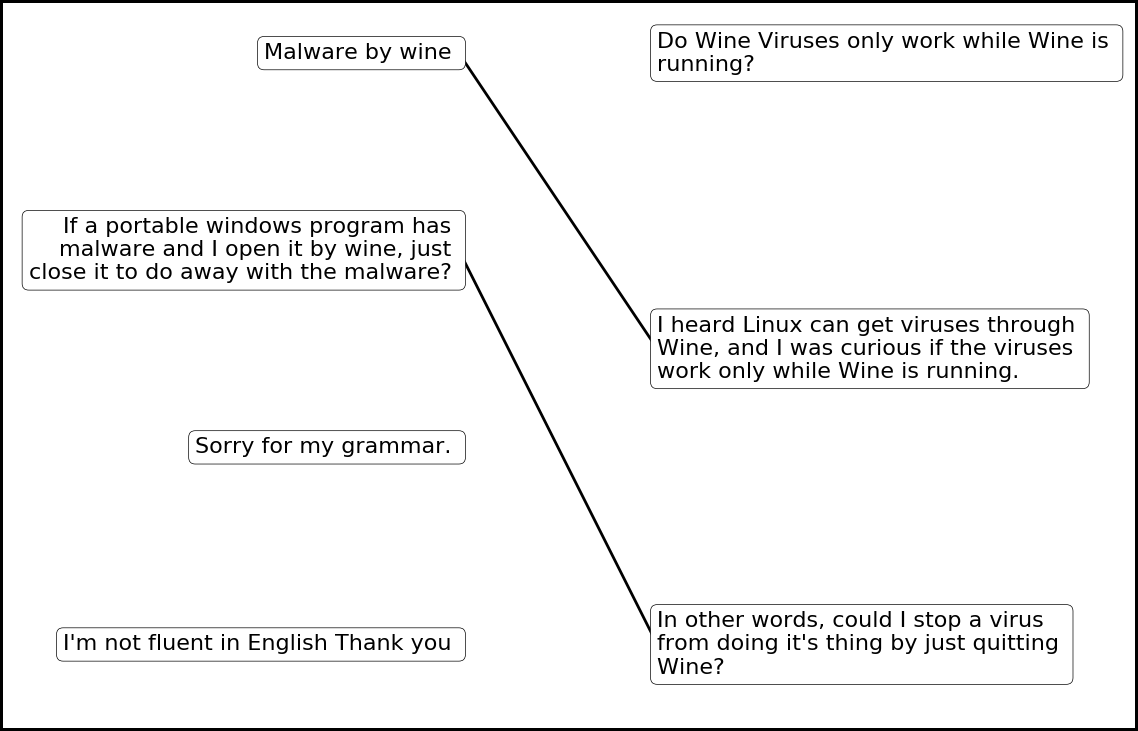}
    \end{tabular}
    \caption{Examples of extracted rationales from the StackExchange dataset using the OT (exact $k=2$) model. Each rationale alignment is displayed visually as lines connecting pairs of sentences from the two text documents.}
    \label{fig:examples}
\end{figure*}
\begin{figure*}[!t]
\centering
    \begin{tabular}{c@{~~~~~~~~~~~~~~~~}c@{~~~~~~~~~~~~~~~~}c}
    \includegraphics[height=2.05in]{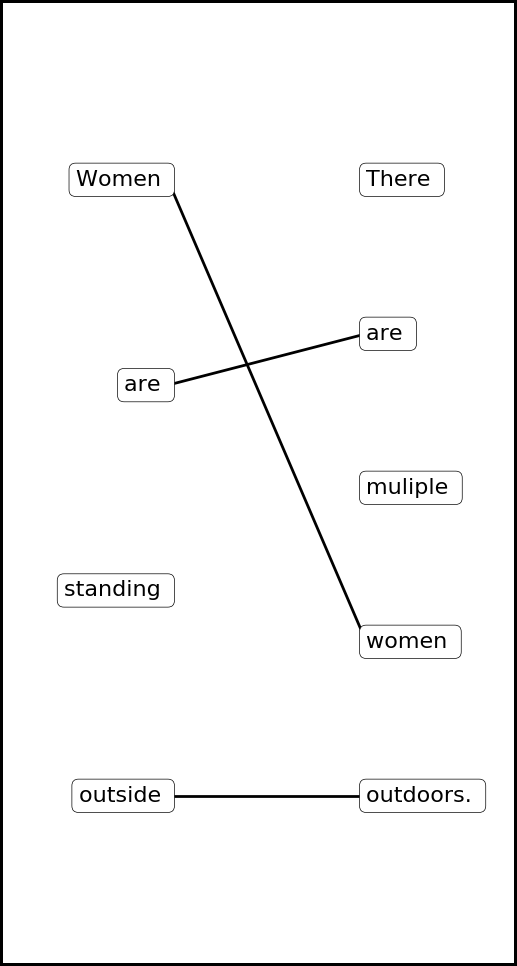}
    &
    \includegraphics[height=2.05in]{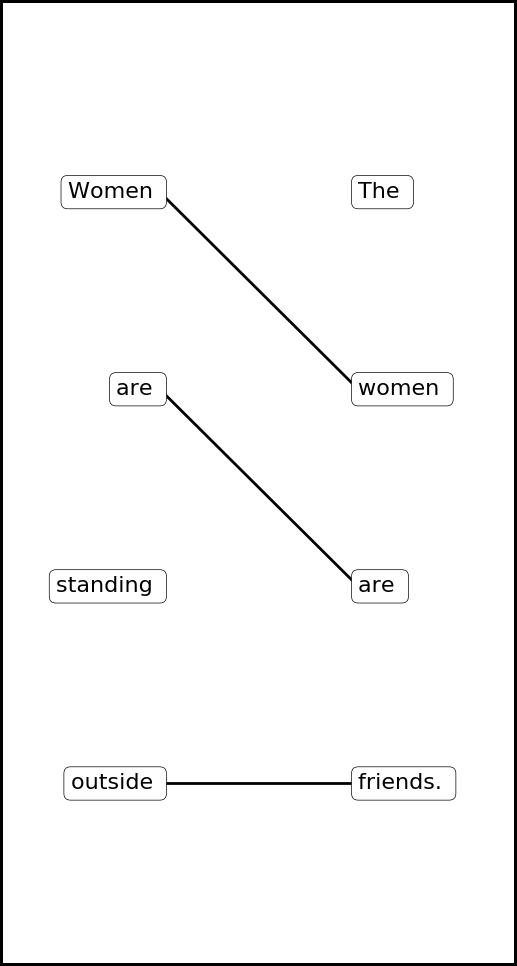}
    &
    \includegraphics[height=2.05in]{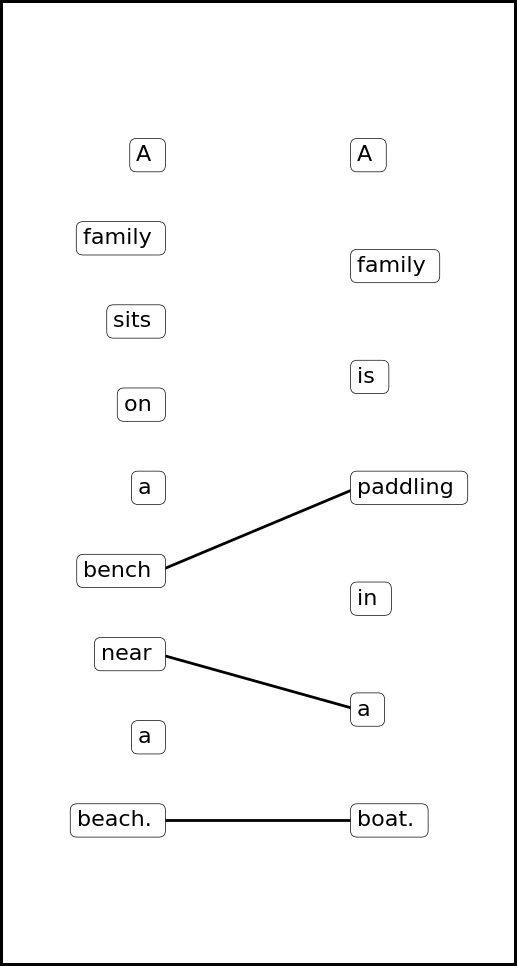}
    \\
    \\
    \includegraphics[height=2.05in]{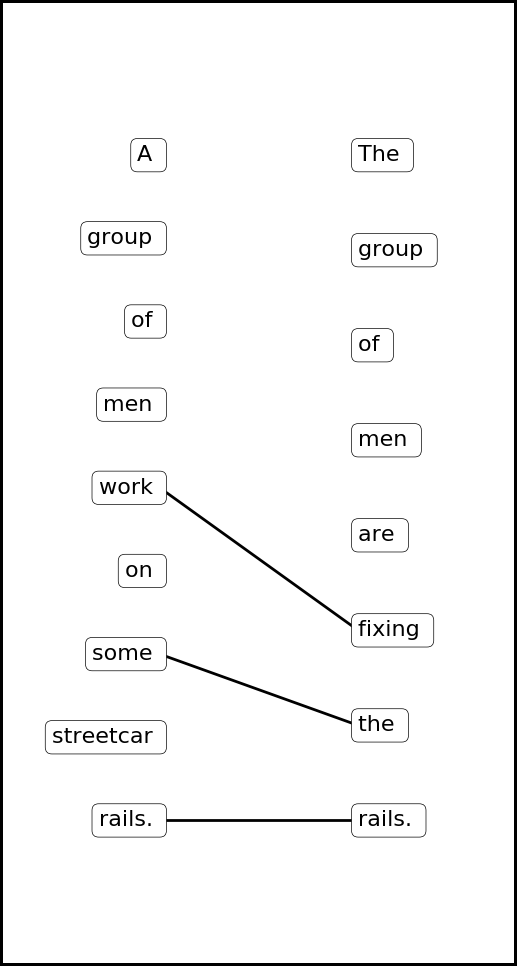}
    &
    \includegraphics[height=2.05in]{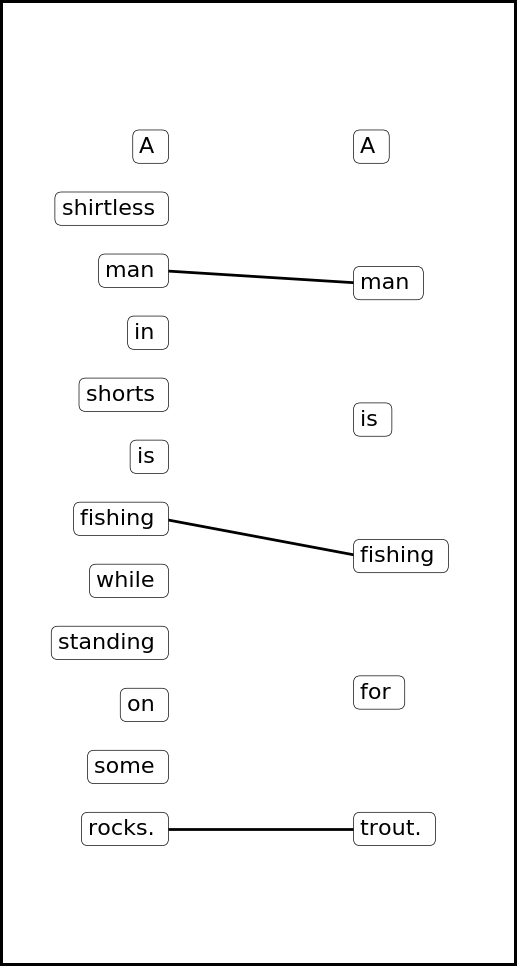}
    &
    \includegraphics[height=2.05in]{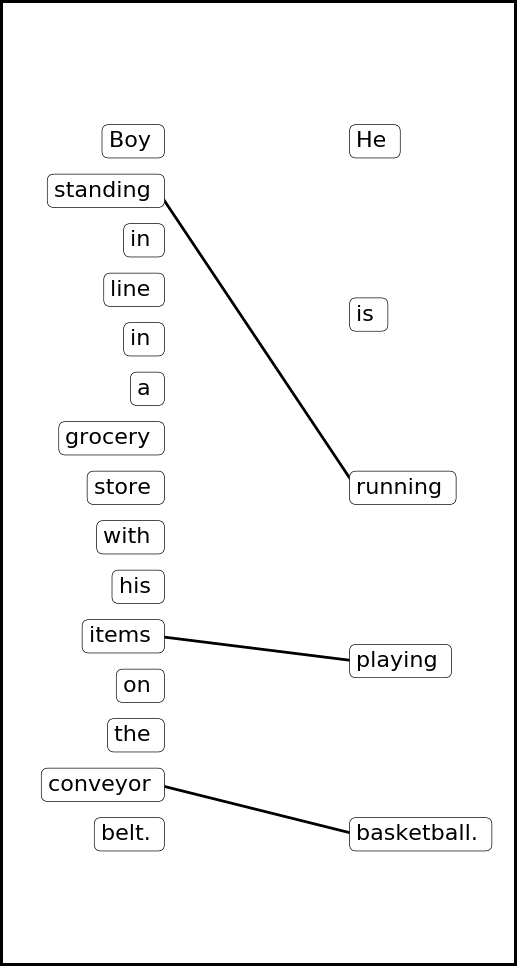}
    \end{tabular}
    \caption{Examples of extracted rationales from the e-SNLI dataset using the OT (exact $k=3$) model. We show two examples of entailment (left column), neutral (middle column) and contradiction (right column).}
    \label{fig:examples_snli}
\end{figure*}

\section{Additional Results}
\label{ssec:mulirc_sru}

\paragraph{MultiRC Experiments with Recurrent Encoder.}
Table \ref{table:multirc-sru} shows the experimental results on the MultiRC dataset when we replace the RoBERTa encoder (results shown in Table \ref{table:multirc}) with the bi-directional simple recurrent unit (SRU) encoder \cite{lei2017simple} that we used for the MultiNews and e-SNLI datasets. In the unsupervised rationale learning setting, the SRU alignment models achieve lower task F1 score and lower rationale token F1 score than the RoBERTa counterpart. Nevertheless, our models still outperform attention-based models, the unsupervised rationale extraction baseline ~\cite{Lei_2016} implemented in ~\citet{deyoung2019eraser}, and even one supervised rationale model ~\cite{lehman} implemented in ~\citet{deyoung2019eraser}. In the supervised rationale learning setting, the SRU alignment models achieve performance comparable to that of the RoBERTa alignment models. Both alignment models achieve higher rationale F1 score than the baseline models, regardless of the encoder architecture, demonstrating the strength of our model for learning rationales. 

\begin{table}[!t!]
\resizebox{0.47\textwidth}{!}{
\centering
\small
\begin{tabular}{l|c|c|c}
\toprule
Model & Task F1 & \% Token & R. F1 \\
\midrule
OT (1:1)         &  59.5    & 20.3   & 24.2   \\
OT (1:2)         &  60.1    & 28.0   & 26.5   \\
OT (relaxed 1:1)         &  59.7   & 13.6   & 19.5   \\
OT (relaxed 1:2)         &  60.2    & 24.7   & 29.1   \\
OT (exact $k=2$)         &  61.0    & 15.2   & 22.7   \\
\midrule
Attention                &  61.4    & 33.2  & 15.7  \\
Attention ($T=0.1$)      & 61.0     & 34.7  & 17.5   \\
Attention ($T=0.01$)     & 61.0    & 34.4  & 18.5    \\
Sparse Attention         & 60.7    & 37.5  & 25.0   \\
\midrule
OT (1:1) (+S)     &  62.1    & 20.5   & 48.1     \\
OT (1:2) (+S)     &  60.0    & 31.3   & 46.0     \\
OT (relaxed 1:1) (+S)         &  60.3    & 18.2   & 46.2   \\
OT (relaxed 1:2) (+S)         &   60.6    & 25.2   & 44.9   \\
OT (exact $k=2$) (+S)     &  61.2    & 16.7   & 48.7     \\
\bottomrule
\end{tabular}}
\caption{MultiRC macro-averaged task F1, percentage of tokens used in active alignments, and token-level F1 of the model-selected rationales compared to human-annotated rationales (R. F1).
(+S) denotes supervised learning of rationales.
All models use a simplified recurrent unit~\cite{lei2017simple} encoder. 
\label{table:multirc-sru}
}
\end{table}

\section{Implementation Details}
\label{ssec:impl_details}

\paragraph{Text Span Extraction.}

Sentences are extracted from the documents using the sentence tokenizer from the \texttt{nltk} Python package\footnote{\url{https://www.nltk.org/}} \cite{nltk}.

\paragraph{Text Embeddings.}

For the bi-directional recurrent encoder, we use pre-trained {\tt fastText} \cite{Bojanowski_2017} word embeddings, while for the RoBERTa encoder, we use its own pre-trained BPE embeddings.

\paragraph{OT Cost Functions.}

We use negative cosine similarity as the cost function for our \textrm{OT (relaxed 1:1)} model to achieve both positive and negative values in the cost matrix.
For all the other OT variants, we use cosine distance, which is non-negative.
We found that cosine-based costs work better than euclidean and dot-product costs for our model.

\paragraph{Sinkhorn Stability.} 

To improve the computational stability of the Sinkhorn algorithm, we use an epsilon scaling trick~\cite{stable_sinkhorn} which repeatedly runs the Sinkhorn iterations with progressively smaller values of epsilon down to a final epsilon of $10^{-4}$.

\paragraph{Loss Function.}
For the document ranking tasks, MultiNews and StackExchange, we train our model using a contrastive loss based on the difference between the optimal transport costs of aligning similar and dissimilar documents. Given a document $D$, if $\mathbf{C}^+$ is the cost matrix between $D$ and a similar document and $\{\mathbf{C}_i^-\}_{i=1}^l$ are the cost matrices between $D$ and $l$ dissimilar documents, then the loss is defined as
\begin{equation*}
    \max_{i \in [[l]]} \left [ \max ( \langle \mathbf{C}^+, \mathbf{P}^+ \rangle - \langle \mathbf{C}_i^-, \mathbf{P}_i^- \rangle + \Delta, 0) \right ],
\end{equation*}
where $\mathbf{P}^+$ and $\mathbf{P}_i^-$ are the OT alignment matrices computed by the Sinkhorn algorithm for $\mathbf{C}^+$ and $\mathbf{C}_i^-$, respectively, and where $\Delta$ is the hinge margin.

For the classification tasks, e-SNLI and MultiRC, we use the standard cross entropy loss applied to the output of a shallow network that processes the cost and alignment matrices.
Specifically, our model implementation is similar to the decomposable attention model~\cite{decomposableattention}, in which the attention-weighted hidden representation is given to a simple 2-layer feed-forward network to generate the classification prediction.
We similarly use the alignment output $\mathbf{P}$ from OT as the weight mask (which will be sparse) to select and average over hidden representations.

\paragraph{Comparison to Human-Annotated Rationales.}

The e-SNLI and MultiRC datasets from the ERASER benchmark provide human rationale annotations, enabling a comparison of model-selected rationales to human-annotated rationales.
However, the rationales are provided independently for each of the two input documents without alignment information.
Therefore, in order to compare our models' rationales to the human annotations, we need to convert our pairwise alignments to independent binary selection rationales for each of the two input documents.
This can be accomplished via thresholding, as described below.

Given an alignment matrix $\mathbf{P} \in \mathbb{R}_+^{n \times m}$ aligning documents $X = \{\mathbf{x}_i\}_{i=1}^n$ and $Y = \{\mathbf{y}_i\}_{i=1}^m$, the goal is to determine two binary rationale selection vectors $\mathbf{R}^\mathbf{x} \in \{0,1\}^n$ and  $\mathbf{R}^\mathbf{y} \in \{0,1\}^m$ indicating which text spans in $X$ and $Y$ are selected.
Each entry of $\mathbf{R}^\mathbf{x}$ and $\mathbf{R}^\mathbf{y}$ is computed as $\mathbf{R}^\mathbf{x}_i = \mathbbm{1}[\sum_{j=1}^m \mathbbm{1}[\mathbf{P}_{i,j} > \delta] > 0]$ and $\mathbf{R}^\mathbf{y}_j = \mathbbm{1}[\sum_{i=1}^n \mathbbm{1}[\mathbf{P}_{i,j} > \delta] > 0]$, where $\mathbbm{1}[\cdot]$ is an indicator function.
Intuitively, this means that $\mathbf{R}^\mathbf{x}_i = 1$ if $\mathbf{P}_{i,j} > \delta$ for any $j=1,\dots,m$, i.e., if at least one text span in $Y$ aligns to text span $\mathbf{x}_i$, and $\mathbf{R}^\mathbf{x}_i = 0$ otherwise.
The meaning is the equivalent for $\mathbf{R}^\mathbf{y}_j$.

The binary selection rationales $\mathbf{R}^\mathbf{x}$ and $\mathbf{R}^\mathbf{y}$ can then be compared against the human-annotated rationales as measured by the F1 score.
The threshold $\delta$ is selected based on the $\delta$ which produces the greatest F1 score on the validation set.

\paragraph{Supervised Rationale Training.}

Our models are designed to learn alignments in an unsupervised manner, but it is possible to alter them to learn from human-annotated rationales in a supervised way.

We do this by constructing a soft version of the independent binary rationale selections described in the previous section.
First, we compute $\widetilde{\mathbf{R}}^\mathbf{x}_i = \sum_{j=1}^m \mathbf{P}_{i,j}$ and $\widetilde{\mathbf{R}}^\mathbf{y}_j = \sum_{i=1}^n \mathbf{P}_{i,j}$ as soft rationale indicators.
We then compute the cross entropy loss $\mathcal{L}_r$ between these soft predictions and the human-annotated rationales.
This loss is combined with the usual task classification cross entropy loss $\mathcal{L}_c$ to form the total loss
\begin{equation*}
    \mathcal{L} = \alpha \cdot \mathcal{L}_c + (1 - \alpha) \cdot \mathcal{L}_r,
\end{equation*}
where $\alpha$ is a hyperparameter. In our experiments, we set $\alpha = 0.2$.

\paragraph{Model Complexity and Speed.}

Table~\ref{table:size_and_speed} compares the model complexity and model speed between OT-based and attention-based models with bi-directional recurrent encoders \cite{lei2017simple}. 
Our model does not add any trainable parameters on top of the text encoder, making it smaller than its attention-based counterparts, which use additional parameters in the attention layer. Our model is 3.3 times slower than attention during training and 1.6 times slower than attention during inference due to the large number of iterations required by the Sinkhorn algorithm for OT. 

\begin{table}[!t!]
\centering
\resizebox{0.47\textwidth}{!}{
\small
\begin{tabular} {l|c|c|c} 
    \toprule
    Model & \#\,Parameters & Train time (s) & Infer time (s) \\
    \midrule
    OT & 2.0M & 600 & 8.0e-3 \\
    Attention & 2.4M & 180 & 4.9e-3 \\
    \bottomrule
\end{tabular}}
\caption{Number of parameters, training time, and inference time for models on the StackExchange dataset. Training time represents training time per epoch while inference time represents the average time to encode and align one pair of documents. All models use an NVIDIA Tesla V100 GPU.}
\label{table:size_and_speed}
\end{table}

\paragraph{Additional Details.}

We use the Adam \cite{Adam} optimizer for training. 
Hyperparameters such as the hinge loss margin, dropout rate, and learning rate are chosen according to the best validation set performance. 
All models were implemented with PyTorch \cite{paszke2017}.
Table \ref{table:size_and_speed} shows model complexity, training time, and inference time for the StackExchange dataset.

\section{Obtaining Permutation Matrix Solutions to Optimal Transport Problems}
\label{ssec:unique_solution}

Our goal in this paper is to create an optimal transport problem that results in an assignment between two sets $X$ and $Y$.
The core idea is to create an expanded optimal transport problem between augmented sets $X'$ and $Y'$ such that $|X'| = |Y'| = n$.
Then Proposition \ref{prop:permutation} implies that the optimal transport problem with $\mathbf{a} = \mathbf{b} = \mathbbm{1}_n / n$ has a permutation matrix solution.
This permutation matrix represents a one-to-one assignment between $X'$ and $Y'$ from which we can extract an assignment between $X$ and $Y$.

However, a problem with this approach is that the permutation matrix solution might not be the only solution.
In general, linear programming problems may have many solutions, meaning we are not guaranteed to find a permutation matrix solution even if it exists.
Since we require a permutation matrix solution in order to obtain our desired sparsity bounds, we are therefore interested in methods for identifying the permutation matrix solution even when other solutions exist.
Although these methods were not necessary for our experiments, since the Sinkhorn algorithm almost always found a permutation matrix solution for our inputs, we present these methods to ensure that the techniques presented in this paper can be used even in cases with degenerate solutions.

One option is to avoid the problem altogether by using cost functions that are guaranteed to produce unique solutions. For example, \citet{brenier} showed that under some normality conditions, the cost function $c(\mathbf{x}, \mathbf{y}) = ||\mathbf{x} - \mathbf{y}||^2$, i.e., the Euclidean distance, produces OT problems with unique solutions. However, it is sometimes preferable to use cost functions with different properties (e.g., bounded range, negative cost, etc.) which may not guarantee a unique OT solution.

To find unique solutions for general cost functions, one method is to first find any solution to the optimal transport problem (e.g., by using the Sinkhorn algorithm) and then to use Birkhoff's algorithm \cite{birkhoff_algorithm} to express that solution as a convex combination of permutation matrices.
Since the original solution is optimal, every permutation matrix that is part of the convex combination must also be optimal (otherwise the cost could be reduced further by removing the suboptimal matrix from the combination and rescaling the others).
Thus we can pick any of the permutation matrices in the convex combination as our optimal permutation matrix solution.
However, since Birkhoff's algorithm is not differentiable, these procedure cannot be used in end-to-end training and can only be applied at inference time.

An alternate method, which preserves the differentiability of our overall approach, is to solve a modified version of the linear programming problem that is guaranteed to have a unique permutation matrix solution that closely approximates the solution the original problem.
Theorem \ref{thm:uniqueness} demonstrates that by adding random iid noise of at most $\epsilon$ to each element of the cost matrix $\mathbf{C}$ to create a new cost matrix $\mathbf{C}^\epsilon$, then with probability one, the resulting linear programming problem on $\mathbf{C}^\epsilon$ has a unique permutation matrix solution $\mathbf{P}^{\epsilon *}$ which costs at most $\epsilon$ more than the true optimal solution $\mathbf{P}^*$.
Thus, we can obtain a permutation matrix solution for $\mathbf{C}$ that is arbitrarily close to optimal.
Furthermore, Corollary \ref{cor:uniqueness} implies that if we know that the difference in cost between the optimal permutation matrix and the second best permutation matrix is $\delta$, then we can choose $\epsilon < \delta$ to ensure that we actually find an optimal permutation matrix.

\begin{theorem}
    \label{thm:uniqueness}
    Consider $L_{\mathbf{C}}(\mathbf{a}, \mathbf{b}) = \underset{\mathbf{P} \in \mathbf{U}(\mathbf{a}, \mathbf{b})}{\operatorname{argmin}} \langle \mathbf{C}, \mathbf{P} \rangle$, where $\mathbf{C} \in \mathbb{R}^{n \times n}$ is arbitrary and $\mathbf{a} = \mathbf{b} = \mathbbm{1}_n / n$.
    Let $\mathbf{E}^\epsilon \in \mathbb{R}^{n \times n}$ be such that $\mathbf{E}^\epsilon_{ij} \overset{iid}{\sim} \mathcal{U}([0, \epsilon])$ where $\epsilon > 0$ and $\mathcal{U}$ is the uniform distribution.
    Define $\mathbf{C}^\epsilon = \mathbf{C} + \mathbf{E}^\epsilon$.
    Let
    \begin{equation*}
        \mathbf{P}^* = \underset{\mathbf{P} \in \mathbf{U}(\mathbf{a}, \mathbf{b})}{\operatorname{argmin}} \langle \mathbf{C}, \mathbf{P} \rangle
    \end{equation*}
    and
    \begin{equation*}
        \mathbf{P}^{\epsilon *} = \underset{\mathbf{P} \in \mathbf{U}(\mathbf{a}, \mathbf{b})}{\operatorname{argmin}} \langle \mathbf{C}^\epsilon, \mathbf{P} \rangle.
    \end{equation*}
    Then
    \begin{enumerate}
        \item $0 \leq \langle \mathbf{C}, \mathbf{P}^{\epsilon *} \rangle - \langle \mathbf{C}, \mathbf{P}^* \rangle \leq \epsilon$.
        \item With probability 1, $\mathbf{P}^{\epsilon *}$ is unique and is a permutation matrix.
    \end{enumerate}
\end{theorem}

\begin{proof}
We begin by proving result 1.
Since $\mathbf{P}^*$ is optimal for $\mathbf{C}$, it must be true that $\langle \mathbf{C}, \mathbf{P} \rangle \leq \langle \mathbf{C}, \mathbf{P}' \rangle$ for any $\mathbf{P}' \in \mathbf{U}(\mathbf{a}, \mathbf{b})$.
As $\mathbf{P}^{\epsilon *} \in \mathbf{U}(\mathbf{a}, \mathbf{b})$, we thus have $\langle \mathbf{C}, \mathbf{P}^* \rangle \leq \langle \mathbf{C}, \mathbf{P}^{\epsilon *} \rangle$ and so $\langle \mathbf{C}, \mathbf{P}^{\epsilon *} \rangle - \langle \mathbf{C}, \mathbf{P}^* \rangle \geq 0$.

To prove the other side of the inequality, first note that for any $\mathbf{P} \in \mathbf{U}(\mathbf{a}, \mathbf{b})$, we have $\langle 
\mathbf{E}^\epsilon, \mathbf{P} \rangle \geq 0$ since $\mathbf{E}^\epsilon_{ij}, \mathbf{P}_{ij} \geq 0$ for all $i,j$.
Combining this with the optimality of $\mathbf{P}^{\epsilon *}$ for $\mathbf{C}^\epsilon$, we can see that
\begin{align*}
    &\langle \mathbf{C}, \mathbf{P}^{\epsilon *} \rangle - \langle \mathbf{C}, \mathbf{P}^* \rangle \\
    &\leq \langle \mathbf{C}, \mathbf{P}^{\epsilon *} \rangle  + \langle \mathbf{E}^\epsilon, \mathbf{P}^{\epsilon *} \rangle - \langle \mathbf{C}, \mathbf{P}^* \rangle \\
    &= \langle \mathbf{C} + \mathbf{E}^\epsilon, \mathbf{P}^{\epsilon *} \rangle - \langle \mathbf{C}, \mathbf{P}^* \rangle \\
    &= \langle \mathbf{C}^\epsilon, \mathbf{P}^{\epsilon *} \rangle - \langle \mathbf{C}, \mathbf{P}^* \rangle \\
    &\leq \langle \mathbf{C}^\epsilon, \mathbf{P}^* \rangle - \langle \mathbf{C}, \mathbf{P}^* \rangle \\
    &= \langle \mathbf{C}^\epsilon - \mathbf{C}, \mathbf{P}^* \rangle \\
    &= \langle \mathbf{C} + \mathbf{E}^\epsilon - \mathbf{C}, \mathbf{P}^* \rangle \\
    &= \langle \mathbf{E}^\epsilon, \mathbf{P}^* \rangle \\
    &\leq \epsilon,
\end{align*}
where the final inequality holds because the entries of $\mathbf{P}^*$ are positive and sum to one and the entries of $\mathbf{E}^\epsilon$ are at most $\epsilon$.
Thus results 1 holds.

Now we will prove result 2.
Since we are solving a linear programming problem over a bounded, convex set $\mathbf{U}(\mathbbm{1}_n / n, \mathbbm{1}_n / n)$, every solution is a convex combination of optimal extremal points.
Thus, a linear program has a unique optimal solution if and only if exactly one of the extremal points is optimal.
By Birkhoff's theorem \cite{birkhoff}, the set of extremal points of $\mathbf{U}(\mathbbm{1}_n / n, \mathbbm{1}_n / n)$ is equal to the set of permutation matrices.
Therefore, if only a single permutation matrix $\mathbf{P}^\sigma$ is optimal for $L_{\mathbf{C}^\epsilon}(\mathbf{a}, \mathbf{b})$, then $\mathbf{P}^\sigma$ is the unique solution.

The goal is thus to show that the event that any two permutation matrices $\mathbf{P}^{\sigma_i}$ and $\mathbf{P}^{\sigma_j}$ corresponding to permutations $\sigma_i \neq \sigma_j$ both solve $L_{\mathbf{C}^\epsilon}(\mathbf{a}, \mathbf{b})$ has probability zero.
The union bound gives
\begin{align*}
    \mathbb{P}(\cup_{\sigma_i \neq \sigma_j}\ \mathbf{P}^{\sigma_i}, \mathbf{P}^{\sigma_j} \textrm{ both solve } L_{\mathbf{C}^\epsilon}(\mathbf{a}, \mathbf{b})&) \\
    \leq \sum_{\sigma_i \neq \sigma_j} \mathbb{P}(\mathbf{P}^{\sigma_i}, \mathbf{P}^{\sigma_j} \textrm{ both solve } L_{\mathbf{C}^\epsilon}(\mathbf{a}, \mathbf{b})&).
\end{align*}
The number of pairs $\sigma_i$ and $\sigma_j$ of distinct permutations of $n$ items is ${n! \choose 2} < \infty$ so the sum is over a finite number of probabilities.
Thus, if we can show that $\mathbb{P}(\mathbf{P}^{\sigma_i}, \mathbf{P}^{\sigma_j} \textrm{ both solve } L_{\mathbf{C}^\epsilon}(\mathbf{a}, \mathbf{b})) = 0$ for any $\sigma_i \neq \sigma_j$, then the sum will also be zero and result 2 will hold.

To show that this is the case, take any two permutations matrices $\mathbf{P}^{\sigma_1}$ and $\mathbf{P}^{\sigma_2}$ for $\sigma_1 \neq \sigma_2$ which are both optimal for $L_{\mathbf{C}^\epsilon}(\mathbf{a}, \mathbf{b})$.
Then it must be true that
\begin{equation*}
    n \langle \mathbf{C}^\epsilon, \mathbf{P}^{\sigma_1} \rangle = n \langle \mathbf{C}^\epsilon, \mathbf{P}^{\sigma_2} \rangle
\end{equation*}
or equivalently 
\begin{equation}
    \label{eq:sumeq}
    n \sum_{i,j = 1}^n \mathbf{C}^\epsilon_{ij} \mathbf{P}^{\sigma_1}_{ij} = n \sum_{k,l = 1}^n \mathbf{C}^\epsilon_{kl} \mathbf{P}^{\sigma_2}_{kl}.
\end{equation}

Let $I^1 \subseteq \{1, \dots, n\} \times \{1, \dots, n\}$ be the indices $(i, j)$ where $\mathbf{P}^{\sigma_1}_{ij} = \frac{1}{n}$ and $\mathbf{P}^{\sigma_2}_{ij} = 0$ and let $I^2 \subseteq \{1, \dots, n\} \times \{1, \dots, n\}$ be the indices $(i, j)$ where $\mathbf{P}^{\sigma_2}_{ij} = \frac{1}{n}$ and $\mathbf{P}^{\sigma_1}_{ij} = 0$.
Thus, for any $(i, j) \notin I^1 \cup I^2$, $P^{\sigma_1}_{ij} = P^{\sigma_2}_{ij}$ and so the terms corresponding to that $(i, j)$ cancel in equation (\ref{eq:sumeq}).
This means that Equation (\ref{eq:sumeq}) can be rewritten as
\begin{equation*}
    n \sum_{i,j \in I^1 \cup I^2} \mathbf{C}^\epsilon_{ij} \mathbf{P}^{\sigma_1}_{ij} = n \sum_{k,l \in I^1 \cup I^2} \mathbf{C}^\epsilon_{kl} \mathbf{P}^{\sigma_2}_{kl}
\end{equation*}
or equivalently, using the definition of $I^1$ and $I^2$, as
\begin{equation*}
    \sum_{i,j \in I^1} \mathbf{C}^\epsilon_{ij} = \sum_{k,l \in I^2} \mathbf{C}^\epsilon_{kl}.
\end{equation*}
Using the definition of $\mathbf{C}^\epsilon$, this becomes
\begin{equation*}
    \sum_{i,j \in I^1} \mathbf{C}_{ij} + \mathbf{E}^\epsilon_{ij} = \sum_{k,l \in I^2} \mathbf{C}_{kl} + \mathbf{E}^\epsilon_{kl}.
\end{equation*}
Grouping terms, we get
\begin{equation*}
    \sum_{i,j \in I^1} \mathbf{E}^\epsilon_{ij} - \sum_{k,l \in I^2} \mathbf{E}^\epsilon_{kl} = \sum_{k,l \in I^2} \mathbf{C}_{kl} - \sum_{i,j \in I^1} \mathbf{C}_{ij}.
\end{equation*}
Since the LHS is a sum/difference of independent continuous random variables and the RHS is a constant, the event that the LHS equals the RHS has probability zero.
Thus, the event that any two permutation matrices $\mathbf{P}^{\sigma_1}$ and $\mathbf{P}^{\sigma_2}$ with $\sigma_1 \neq \sigma_2$ are both optimal for $L_{\mathbf{C}^\epsilon}(\mathbf{a}, \mathbf{b})$ has probability zero.

\end{proof}

\begin{corollary}
\label{cor:uniqueness}

If $\langle \mathbf{C}, \mathbf{P}^\sigma \rangle - \langle \mathbf{C}, \mathbf{P}^* \rangle = 0$ or $\langle \mathbf{C}, \mathbf{P}^\sigma \rangle - \langle \mathbf{C}, \mathbf{P}^* \rangle > \epsilon$ for every permutation matrix $\mathbf{P}^\sigma$, then the permutation matrix $\mathbf{P}^{\epsilon *}$ is an exact solution to $L_{\mathbf{C}}(\mathbf{a}, \mathbf{b})$.
\end{corollary}

\begin{proof}
Theorem \ref{thm:uniqueness} says that that $\langle \mathbf{C}, \mathbf{P}^{\epsilon *} \rangle - \langle \mathbf{C}, \mathbf{P}^* \rangle \leq \epsilon$.
Since $\mathbf{P}^{\epsilon *}$ is a permutation matrix, the assumptions in this corollary thus imply that that $\langle \mathbf{C}, \mathbf{P}^{\epsilon *} \rangle - \langle \mathbf{C}, \mathbf{P}^* \rangle = 0$, meaning $\mathbf{P}^{\epsilon *}$ is an exact solution to $L_{\mathbf{C}}(\mathbf{a}, \mathbf{b})$.
\end{proof}

\end{document}